%% file: main.tex
\documentclass[letterpaper]{article} 
\usepackage{aaai2027}  
\nocopyright
\usepackage[hyphens]{url}  
\usepackage{graphicx} 
\usepackage{natbib}  
\usepackage{caption} 
\usepackage{algorithm}
\usepackage{algorithmic}
\usepackage{amsthm}
\usepackage{amsmath}
\usepackage{amssymb}
\usepackage{complexity}
\usepackage{multirow,booktabs}
\usepackage[inline]{enumitem}

\usepackage{subcaption}
\usepackage{xcolor}
\usepackage{tikz}
\usetikzlibrary{shapes, backgrounds, calc, positioning, arrows.meta, decorations.pathmorphing, bending}

\theoremstyle{definition}
\newtheorem{problem}{Problem}
\newtheorem{definition}{Definition}
\newtheorem{theorem}{Theorem}

\newtheorem{lemma}{Lemma}
\newtheorem{proposition}{Proposition}
\theoremstyle{remark}

\newtheorem{example}{Example}
\newtheorem*{claim}{Claim}

\newfloat{algorithmSM}{htb!}{loa}
\floatname{algorithmSM}{Algorithm}

\definecolor{PathBlue}{RGB}{0, 76, 153}
\definecolor{PathRed}{RGB}{192, 0, 0}
\definecolor{PathOrange}{RGB}{255, 128, 0}
\definecolor{PathGreen}{RGB}{0, 153, 0}

\newcommand{\mc}[1]{\mathcal{#1}}
\renewcommand{\G}{\mc{G}} 
\newcommand{\V}{\mc{V}} 
\renewcommand{\E}{\mc{E}} 
\newcommand{\graph}{\G=(\V, \E)}
\renewcommand{\A}{\mc{A}} 
\renewcommand{\o}{\mathbf{o}}
\renewcommand{\c}{\mathbf{c}}
\renewcommand{\S}{\mc{S}}
\renewcommand{\P}{\mc{P}}
\newcommand{\I}{\mc{I}}

\newcommand{\xp}[2]{x_{#1,#2}^{\P}}
\newcommand{\xs}[2]{x_{#1,#2}^{\S}}
\newcommand{\yp}[3]{y_{#1,#2,#3}^{\P}}
\newcommand{\ys}[3]{y_{#1,#2,#3}^{\S}}

\newcommand{\tmin}[1]{t_{#1}}
\newcommand{\dmin}[1]{\delta_{#1}}
\newcommand{\dminPar}[1]{\delta(#1)}
\newcommand{\lb}[1]{l_{#1}}

\newcommand{\phiT}{\Phi(T)}

\usepackage{newfloat}
\usepackage{listings}
\DeclareCaptionStyle{ruled}{labelfont=normalfont,labelsep=colon,strut=off} 
\floatstyle{ruled}
\newfloat{listing}{tb}{lst}{}
\floatname{listing}{Listing}

\usepackage{booktabs}

\title{Pivot-and-Station Multi-Agent Path Finding:\\ Solvability, Complexity, and Algorithms}
\author {
    Andrea Di Nezza\textsuperscript{\rm 1}\equalcontrib,
    Mihir Patel\textsuperscript{\rm 2}\equalcontrib,
    Fabio Fagnani\textsuperscript{\rm 1},
    Sara Bernardini\textsuperscript{\rm 2}
}
\affiliations {

    \textsuperscript{\rm 1}Department of Mathematical Sciences, Politecnico di Torino, Turin, Italy\\
    \textsuperscript{\rm 2}Department of Computer Science, University of Oxford, Oxford, United Kingdom
}

\begin{document}

\maketitle

\begin{abstract}
    Automated high-density storage systems (warehouses, robotic parking, plant logistics, etc.) require fleets of agents to move through scarce task-critical resources and then park without obstructing future operations. We introduce \emph{Pivot-and-Station Multi-Agent Path Finding (PS-MAPF)}, a MAPF variant in which a subset of tasked agents must each visit one of a set of interchangeable pivots (e.g., workstations) before the entire fleet terminates at anonymous stations, one agent per station. We characterize solvability completely: every instance on a 2-edge-connected graph is solvable, and, on arbitrary connected graphs, a structural effective-distance measure relative to the number of unoccupied vertices gives a necessary and sufficient condition. We prove that minimizing station-makespan or station-flowtime is NP-hard already with a single pivot. We present three algorithms, a complete baseline, a SAT-based optimal solver, and Pivot-Prioritized Planning (PPP), the last solving 74-89\% of benchmark instances with makespan and flowtime orders of magnitude below the baseline.
    \end{abstract} 

\input{Sections/introduction.tex}
\input{Sections/prob_statement.tex}
\input{Sections/existence.tex}
\input{Sections/complexity.tex}
\input{Sections/algorithms.tex}
\input{Sections/experimental.tex}

\input{Sections/conclusions.tex}

\bibliography{aaai2027}


\end{document}

%% file: Sections/introduction.tex
\section{Introduction and Related Work}
Infrastructure-based high-density automated storage systems (warehouses, robotic parking, plant logistics; e.g., Ocado, Amazon Sequoia, Versatile, Robotic Parking, WPS) move fleets of \emph{agents} (robots, shuttles, or motorized cells) through a fixed mechanical topology of rails, lifts, conveyors, or grid cells. At any moment, some agents are moving toward a \emph{pivot}: a task-critical resource that must be visited, such as a workstation, a pickup or drop-off point, or a charging interface; in a given system, pivots form a single class of functionally interchangeable resources (systems with several classes give rise to typed pivots; we treat one). The other agents, having visited a workstation or not needing one, must park at a \emph{station}: an admissible terminal position (storage cell, staging area) from which they remain available for future tasks without obstructing traffic; any agent may terminate at any station, one agent per station. Poor terminal placement propagates congestion, and poorly coordinated pivot access causes active agents to queue; the two should therefore be optimized \emph{jointly}: a locally good route to a pivot may leave agents obstructing others' paths or blocking terminal placements, while convenient station assignments may ignore upcoming access to scarce resources.

We call the combined problem \emph{Pivot-and-Station Multi-Agent Path Finding (PS-MAPF)}, a new variant of Multi-Agent Path Finding (MAPF) \cite{stern2019mapf}: on a graph whose vertices include pivots and stations, a subset of agents is \emph{tasked} and must visit at least one pivot (any pivot serves any tasked agent) before reaching its terminal position; the remaining \emph{untasked} agents have no pivot requirement but still occupy space and interact with tasked agents throughout. Every agent must terminate at a distinct station, and a solution is a set of collision-free paths satisfying both the pivot-access and terminal-station requirements.

In classical MAPF \cite{stern2019mapf}, feasibility is polynomial-time decidable via pebble motion \cite{kornhauser1984pebble}, whereas optimizing makespan or flowtime is \NP-hard \cite{yu2013structure}. Optimal solvers are search-based, e.g., CBS \cite{sharon2015cbs}, or reduction-based, encoding MAPF as SAT \cite{surynek2016efficient}, which often dominates in dense settings \cite{acha2022multi, surynek2019unifying}; scalable suboptimal methods include rule-based algorithms \cite{luna2011push, surynek2009novel}, prioritized planning \cite{erdmann1987multiple}, whose modern variants tune agent orderings \cite{silver2005cooperative, ma2019searching}, the near-optimal LaCAM* \cite{okumura2024engineering}, and MAPF-HD \cite{makino2025mapfhd} for high-density settings. Our optimal solver builds on the SAT line, and our fast method is a prioritized planner.

PS-MAPF captures a coordination structure not fully addressed by existing variants \cite{ma2017overview}. In Anonymous MAPF (AMAPF), destinations are unassigned and makespan optimization is polynomial \cite{yu2013multiagent}; in Target Assignment and Path Finding (TAPF), destinations are interchangeable only within subgroups, which is \NP-hard \cite{ma2016tapf}; neither models mandatory access to scarce intermediate resources. Generalized TAPF \cite{nguyen2019generalised} and Multi-Goal MAPF \cite{surynek2021multi} assign each agent or group a specific sequence of targets, whereas PS-MAPF treats pivots as unassigned, interchangeable resources. Multi-Agent Combinatorial Path Finding (MCPF) \cite{ren2021ms, ren2022cbss} is the nearest, combining intermediate targets with anonymous terminal assignment, but the requirement points the opposite way (in MCPF every target must be visited by some agent; in PS-MAPF every tasked agent must visit some pivot, and a pivot no agent needs is never visited). MCPF treats agents symmetrically, while PS-MAPF distinguishes tasked/untasked agents and places the entire fleet. MAPF with Unassigned Agents \cite{felner2026mapfua} distinguishes tasked and untasked agents but does not require anonymous terminal reconfiguration of the whole fleet. Multi-Agent Pickup and Delivery \cite{ma2017lifelong} shares the intermediate requirement but is lifelong with assigned pickup and delivery locations, whereas PS-MAPF is one-shot with anonymous pivots and stations.

This paper makes the following contributions.
\begin{enumerate*}[label=(\roman*)]
\item \emph{A new MAPF variant}, PS-MAPF, motivated by high-density automated storage systems (Sec.~\ref{sec:ps-mapf}).
\item \emph{A complete characterization of solvability} checkable in polynomial time: sufficiency via 2-edge connectivity and, on arbitrary connected graphs, a necessary and sufficient condition comparing an effective-distance measure with the number of unoccupied vertices (Sec.~\ref{sec:existence}).
\item \emph{A sharp complexity threshold}: minimizing station-makespan and station-flowtime is \NP-hard already for a single pivot (Sec.~\ref{sec:complexity}).
\item \emph{Algorithms and experiments}: a complete baseline derived from our constructive solvability proof, a reduction-based optimal solver, and Pivot-Prioritized Planning (PPP), incomplete but fast (prioritized pivot routing plus flow-based station assignment), solving 74--89\% of benchmark instances at orders-of-magnitude lower cost than the baseline (Sec.~\ref{sec:algorithm} and~\ref{sec:experimental_study}).
\end{enumerate*}

%% file: Sections/prob_statement.tex
\section{Problem Statement}
\label{sec:ps-mapf}
We consider a connected, undirected graph $\graph$ without self-loops, with a set of \emph{pivots} $\P \subseteq \V$, a set of \emph{stations} $\S \subseteq \V$, and a set $\A$ of agents moving on the graph. A \emph{configuration} is an injective function $\c: \A \to \V$ assigning to each agent a distinct vertex (hence $|\V| \geq |\A|$). Agents' initial positions are specified by a configuration $\o: \A \to \V$, called the \emph{origin configuration}; we refer to $\o(i)$ as the \emph{origin} of agent~$i$. Agents are partitioned into \emph{tasked agents} $\A^t$, which must pass through one of the pivots in $\P$ before reaching a station, and \emph{untasked agents} $\A^u$, which can go directly to a station. Any pivot serves any tasked agent, and any agent may terminate at any station.

	\begin{definition}
		\label{def:matrix}
		Given a graph $\graph$, a set of pivots $\P \subseteq \V$, a set of stations $\S \subseteq \V$, a set of tasked agents $\A^t$, a set of untasked agents $\A^u$, and an origin function $\o:\A\to \V$, where $\A = \A^t \cup \A^u$, $\A^t \cap \A^u = \emptyset$, $\o$ is injective and $|\S| \geq |\A|$,
		we define a \emph{PS-MAPF Instance} as the tuple
		$\I = \left(\G, \P, \S, \A^t, \A^u, \o \right)$.
	\end{definition}
	
	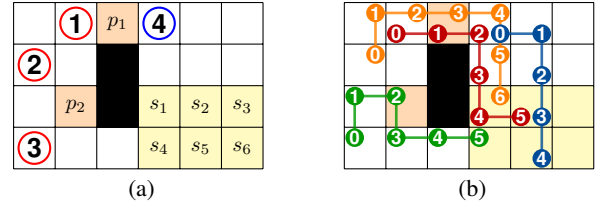
\begin{figure}[htb!]
		\centering
		\begin{subfigure}{\dimexpr 48\columnwidth/100 \relax}
			\centering
			\begin{tikzpicture}[scale=55/100]
				\fill[yellow!30] (3,0) rectangle (6,2); 
				\fill[orange!30] (2,3) rectangle (3,4); 
				\fill[orange!30] (1,1) rectangle (2,2); 
				\fill[black] (2,1) rectangle (3,3); 
				\draw[step=1cm, black, thin] (0,0) grid (6,4); 
				
				\node[scale=0.8, font=\sffamily\normalsize] at (2.5, 3.5) {$p_1$};
				\node[scale=0.8, font=\sffamily\normalsize] at (1.5, 1.5) {$p_2$};
				\node[scale=0.8, font=\sffamily\normalsize] at (3.5, 1.5) {$s_1$};
				\node[scale=0.8, font=\sffamily\normalsize] at (4.5, 1.5) {$s_2$};
				\node[scale=0.8, font=\sffamily\normalsize] at (5.5, 1.5) {$s_3$};
				\node[scale=0.8, font=\sffamily\normalsize] at (3.5, 0.5) {$s_4$};
				\node[scale=0.8, font=\sffamily\normalsize] at (4.5, 0.5) {$s_5$};
				\node[scale=0.8, font=\sffamily\normalsize] at (5.5, 0.5) {$s_6$};
				
				\node[draw=red, circle, thick, fill=white, inner sep=1pt, font=\sffamily\bfseries] at (3/2, 7/2) {1};
				\node[draw=red, circle, thick, fill=white, inner sep=1pt, font=\sffamily\bfseries] at (1/2, 5/2) {2};
				\node[draw=red, circle, thick, fill=white, inner sep=1pt, font=\sffamily\bfseries] at (1/2, 1/2) {3};
				
				\node[draw=blue, circle, thick, fill=white, inner sep=1pt, font=\sffamily\bfseries] at (7/2, 7/2) {4};
			\end{tikzpicture}
			\vspace{-1mm}
			\caption{}
			\label{fig:matrix_instance_example}
		\end{subfigure}
		\hfill
		\begin{subfigure}{\dimexpr 48\columnwidth/100 \relax}
			\centering
			\begin{tikzpicture}[scale=55/100]
				\fill[yellow!30] (3,0) rectangle (6,2); 
				\fill[orange!30] (2,3) rectangle (3,4); 
				\fill[orange!30] (1,1) rectangle (2,2); 
				\fill[black] (2,1) rectangle (3,3); 
				\draw[step=1cm, black, thin] (0,0) grid (6,4); 
				
				\tikzset{
					pathstyle/.style={-, >=stealth, line width=1pt, rounded corners=1pt, opacity=80/100}, 
					timemarker/.style={circle, fill=#1, text=white, inner sep=0pt, minimum size=7pt, font=\sffamily\bfseries\fontsize{7}{7}\selectfont, thin}
				}
				
				\def\sfas{1/4}
				
				\draw[pathstyle, PathRed] 
				([shift={(-\sfas,-\sfas)}]3/2, 7/2) -- 
				([shift={(-\sfas,-\sfas)}]5/2, 7/2) -- 
				([shift={(-\sfas,-\sfas)}]7/2, 7/2) -- 
				([shift={(-\sfas,-\sfas)}]7/2, 5/2) -- 
				([shift={(-\sfas,-\sfas)}]7/2, 3/2) -- 
				([shift={(-\sfas,-\sfas)}]9/2, 3/2);
				\node[timemarker=PathRed] at ([shift={(-\sfas,-\sfas)}]3/2, 7/2) {0};
				\node[timemarker=PathRed] at ([shift={(-\sfas,-\sfas)}]5/2, 7/2) {1};
				\node[timemarker=PathRed] at ([shift={(-\sfas,-\sfas)}]7/2, 7/2) {2};
				\node[timemarker=PathRed] at ([shift={(-\sfas,-\sfas)}]7/2, 5/2) {3};
				\node[timemarker=PathRed] at ([shift={(-\sfas,-\sfas)}]7/2, 3/2) {4};
				\node[timemarker=PathRed] at ([shift={(-\sfas,-\sfas)}]9/2, 3/2) {5};
				
				\draw[pathstyle, PathOrange] 
				([shift={(\sfas,\sfas)}]1/2, 5/2) -- 
				([shift={(\sfas,\sfas)}]1/2, 7/2) --       
				([shift={(\sfas,\sfas)}]3/2, 7/2) --       
				([shift={(\sfas,\sfas)}]5/2, 7/2) --       
				([shift={(\sfas,\sfas)}]7/2, 7/2) -- 
				([shift={(\sfas,\sfas)}]7/2, 5/2) --      
				([shift={(\sfas,\sfas)}]7/2, 3/2);        
				\node[timemarker=PathOrange] at ([shift={(\sfas,\sfas)}]1/2, 5/2) {0};
				\node[timemarker=PathOrange] at ([shift={(\sfas,\sfas)}]1/2, 7/2) {1};
				\node[timemarker=PathOrange] at ([shift={(\sfas,\sfas)}]3/2, 7/2) {2};
				\node[timemarker=PathOrange] at ([shift={(\sfas,\sfas)}]5/2, 7/2) {3};
				\node[timemarker=PathOrange] at ([shift={(\sfas,\sfas)}]7/2, 7/2) {4}; 
				\node[timemarker=PathOrange] at ([shift={(\sfas,\sfas)}]7/2, 5/2) {5};
				\node[timemarker=PathOrange] at ([shift={(\sfas,\sfas)}]7/2, 3/2) {6};
				
				\draw[pathstyle, PathBlue] 
				([shift={(\sfas,-\sfas)}]7/2, 7/2) -- 
				([shift={(\sfas,-\sfas)}]9/2, 7/2) --       
				([shift={(\sfas,-\sfas)}]9/2, 5/2) --       
				([shift={(\sfas,-\sfas)}]9/2, 3/2) -- 
				([shift={(\sfas,-\sfas)}]9/2, 1/2);                              
				\node[timemarker=PathBlue] at ([shift={(\sfas,-\sfas)}]7/2, 7/2) {0}; 
				\node[timemarker=PathBlue] at ([shift={(\sfas,-\sfas)}]9/2, 7/2) {1};
				\node[timemarker=PathBlue] at ([shift={(\sfas,-\sfas)}]9/2, 5/2) {2};
				\node[timemarker=PathBlue] at ([shift={(\sfas,-\sfas)}]9/2, 3/2) {3};
				\node[timemarker=PathBlue] at ([shift={(\sfas,-\sfas)}]9/2, 1/2) {4};
				
				\draw[pathstyle, PathGreen] 
				([shift={(-\sfas,\sfas)}]1/2, 1/2) -- 
				([shift={(-\sfas,\sfas)}]1/2, 3/2) -- 
				([shift={(-\sfas,\sfas)}]3/2, 3/2) -- 
				([shift={(-\sfas,\sfas)}]3/2, 1/2) -- 
				([shift={(-\sfas,\sfas)}]5/2, 1/2) -- 
				([shift={(-\sfas,\sfas)}]7/2, 1/2);
				\node[timemarker=PathGreen] at ([shift={(-\sfas,\sfas)}]1/2, 1/2) {0};
				\node[timemarker=PathGreen] at ([shift={(-\sfas,\sfas)}]1/2, 3/2) {1};
				\node[timemarker=PathGreen] at ([shift={(-\sfas,\sfas)}]3/2, 3/2) {2};
				\node[timemarker=PathGreen] at ([shift={(-\sfas,\sfas)}]3/2, 1/2) {3};
				\node[timemarker=PathGreen] at ([shift={(-\sfas,\sfas)}]5/2, 1/2) {4};
				\node[timemarker=PathGreen] at ([shift={(-\sfas,\sfas)}]7/2, 1/2) {5};
			\end{tikzpicture}
			\vspace{-1mm}
			\caption{}
			\label{fig:matrix_solution_example}
		\end{subfigure}
		
		\caption{Example of a PS-MAPF instance $\I$ and its solution $\Pi$ on a grid graph with black obstacles. 
			Pivots $\P=\{p_1, p_2\}$ and storage destinations $\S=\{s_1,..,s_6\}$ are highlighted in orange and yellow, respectively.
			In \ref{fig:matrix_instance_example}, the problem instance, showing the labeled positions of pivots and destinations, along with the starting positions of tasked agents (red circles 1, 2, 3) and the untasked agent (blue circle 4).
			In \ref{fig:matrix_solution_example}, the corresponding solution, showing the non-conflicting paths $\Pi=(\pi_1,\pi_2,\pi_3,\pi_4)$ toward their targets, with time markers indicating the agent positions at each time step.}
		\label{fig:full_example}
	\end{figure}
	
	\begin{example}
		\textit{Figure~\ref{fig:matrix_instance_example} illustrates instance $\I = \left(\G, \P, \S, \A^t, \A^u, \o \right)$, where $\G$ is a grid graph with black cells representing obstacles; pivots $\P = \{p_1, p_2\}$ (orange), stations $\S = \{s_1, .., s_6\}$ (yellow), tasked agents $\A^t = \{1, 2, 3\}$ (red circles) and untasked agents $\A^u = \{4\}$ (blue circle), with each agent starting at the indicated position.}
	\end{example}

	A \emph{path} in $\graph$ from $x\in\V$ to $y\in\V$ is a function $\pi:[0..T]\to \V$, $T \in \mathbb{N}$, with $\pi(0)=x$, $\pi(T)=y$, and, for every $t\in [0..T-1]$, $(\pi(t),\pi(t+1))\in\E$ (moving) or $\pi(t+1)=\pi(t)$ (waiting); $T$ is the path's \emph{length}.
	The \emph{distance} $d(x,y)$ between two vertices $x$ and $y$ is the minimum length among all paths between them. Given $\mc W\subseteq \mc V$, the \emph{subgraph induced by} $\mc W$ is $\mc G[\mc W]=(\mc W, \mc E\cap (\mc W\times\mc W))$.

	Movements must avoid two types of \emph{conflicts} between paths $\pi_i,\pi_j:[0..T]\to \V$. A \emph{vertex conflict} occurs at time $t\in[0..T]$ if $\pi_i(t)=\pi_j(t)$; a \emph{swapping conflict} occurs at time $t\in[0..T-1]$ if $\pi_i(t)=\pi_j(t+1)$ and $\pi_i(t+1)=\pi_j(t)$, with $\pi_i(t)\neq \pi_i(t+1)$ and $\pi_j(t)\neq \pi_j(t+1)$ (the two agents traverse the same edge in opposite directions).
	Paths $\pi_i$ and $\pi_j$ are \emph{conflict-free} if neither conflict occurs at any time step; the definitions extend to paths of different lengths by padding the shorter one with wait actions.

	For an agent's path $\pi:[0..T]\to \V$, the \emph{time-at-pivot} is the first time at which the agent reaches a pivot, $t_\P(\pi) = \min \{t \in [0..T] \mid \pi(t) \in \P\}$, and the \emph{time-at-station} is the first time at which the agent reaches a station vertex and never leaves it, $t_\S(\pi)=\min\{t\in [0..T]\mid \exists q\in\S\ \forall \tau\in[t.. T],\ \pi(\tau)=q\}$; for both, $\min\emptyset=\infty$.
	An untasked agent $i\in \A^u$'s path $\pi$ is \emph{successful} if $\pi(0)=\o(i)$ and $t_\S(\pi)<\infty$; a tasked agent $i\in \A^t$'s path $\pi$ is \emph{successful} if $\pi(0)=\o(i)$, $t_\P(\pi)<\infty$ and $t_\S(\pi)<\infty$.

\begin{definition}
\label{def:solution}
Given a PS-MAPF instance $\I = \left(\G, \P, \S, \A^t, \A^u, \o \right)$, a tuple of paths $\Pi = \{\pi_i\}_{i\in \A}$, all of the same length $T$, is called a \emph{feasible motion} of length $T$ if all paths are pairwise conflict-free, and a \emph{successful motion} if, moreover, all paths are successful.
A successful motion $\Pi$ is also called a \emph{solution} for $\I$.
We denote the \emph{set of all solutions} for $\I$ by $Sol(\I)$.
\end{definition}

Given a feasible motion $\Pi = \{\pi_i\}_{i \in \A}$ of length $T$, $\Pi(t)$ is the configuration $i \mapsto \pi_i(t)$ (well-defined since $\Pi$ is conflict-free); a feasible motion need not start at $\o$. 

	\begin{example}
	\textit{Figure~\ref{fig:matrix_solution_example} depicts a solution $\Pi = (\pi_1,\pi_2,$ $\pi_3,\pi_4)$ for the instance $\I$ introduced in Figure~\ref{fig:matrix_instance_example}. The non-conflicting paths are illustrated by the colored trajectories, with time markers indicating the time steps. Tasked agents 1, 2, and 3 successfully reach a pivot with $t_\P$ values of 1, 3, and 2, respectively, before reaching their storage destinations with $t_\D$ values of 5, 6, and 5, respectively. The untasked agent 4 reaches its storage destination with a $t_\S$ value of 4.}
	\end{example}

Below, we define two cost functions based on the time all agents reach their stations.

	\begin{definition}
		\label{def:solution-cost-dest}
		Given a PS-MAPF instance $\I$ and a solution $\Pi = \{\pi_i\}_{i\in \A}$, the \emph{Station-Makespan} and the \emph{Station-Flowtime} of $\Pi$ are
		$M_\S^{\I}(\Pi)=\max_{i\in A} t_\S(\pi_i)$ and
		$\Sigma_\S^{\I}(\Pi)=\sum_{i\in A} t_\S(\pi_i)$.
	\end{definition}
	
	\begin{example}
\textit{The example solution in Figure~\ref{fig:matrix_solution_example} has $M_\S^{\I}(\Pi)=6$ and $\Sigma_\S^{\I}(\Pi)=20$, which are all minimal for this instance.}
	\end{example}

We aim to find optimal solutions w.r.t.\ these functions.

%% file: Sections/existence.tex
\section{Solvability}
\label{sec:existence}
We analyze when a PS-MAPF instance admits a solution. When no agent is tasked ($\A^t=\emptyset$ and $\A=\A^u$), PS-MAPF coincides with \emph{anonymous} MAPF~\cite{yu2013multiagent}, and we call the instance \emph{pivot-free}. The following proposition follows from classical results on the \emph{unlabeled pebble motion} problem~\cite{kornhauser1984pebble}.

\begin{proposition}
	\label{prop:reachability}
	A pivot-free PS-MAPF instance $\I = \left(\G, \P, \S, \emptyset, \A^u, \o \right)$ always admits a solution ($Sol(\I)\neq\emptyset$).
\end{proposition}

Prop.~\ref{prop:reachability} implies that solvability never depends on where the stations lie: once the tasked agents have completed their pivot visits, routing the whole fleet onto $\S$ is a pivot-free instance. Only the pivot requirement can obstruct solvability, e.g., on a path graph, no two agents can ever exchange their relative order, so 
if a single pivot is located at an endpoint, and there are two tasked agents, the farther one can never reach the pivot, no matter how many vertices are unoccupied (see Figure \ref{fig:existence}). 

\begin{figure}[htbp]
\centering
\includegraphics[height=0.4cm]{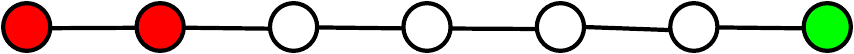}
\caption{An unsolvable instance: red vertices are origins of tasked agents, the green vertex is the only pivot, every vertex is a station.}
\label{fig:existence}
\end{figure}

To present the results, we first introduce some preliminary notation. 
A path $\gamma:[0..T]\to \V$ is called \emph{simple} if $\gamma(t)\neq \gamma(t')$ for every $t, t'\in [0..T]$ with $t\neq t'$ and $\{t, t'\}\neq \{0, T\}$. 
A simple path is called a \emph{cycle} if $T>2$ and $\gamma(0)=\gamma(T)$. Often, a cycle $\gamma$ is identified with the set of edges $\{(\gamma(t), \gamma(t+1))\,|\, t=[0..T-1]\}$.
We can associate every cycle $\gamma:[0..T]\to \V$ with $T\geq 3$ with a feasible length-one motion \emph{rotation}, $\Pi^\gamma=\{\pi^\gamma_i\}_{i\in \A}$, defined by a synchronized rotation of all agents present in the vertices of the cycle. 
Formally, for $i\in\A$ in position $\gamma(t_i)$ for some $t_i\in[0..T-1]$, we put $(\pi^\gamma_i(0), \pi^\gamma_i(1))= (\gamma(t_i), \gamma(t_i+1))$. For the remaining $i\in\A$ whose location is outside the cycle, we instead put $\pi^\gamma_i(0)=\pi^\gamma_i(1)=\o(i)$. Finally, we introduce two operations on motions. If $\Pi= \{\pi_i\}_{i\in \A}$ is a motion of length $T$, the \emph{inversion} of $\Pi$ is the motion $\Pi^-= \{\pi^-_i\}_{i\in \A}$ of length $T$  where $\pi^-_i(t)=\pi_i(T-t)$ for $t=0..T$. We also have a natural concept of motion concatenation.
Two tuples of paths $\Pi^1= \{\pi^1_i\}_{i\in \A}$ and $\Pi^2= \{\pi^2_i\}_{i\in \A}$, with $\pi_i^r:[0..T_r]\to\V$ for $i\in\A$ and $r=1,2$, are called \emph{concatenable} if $\pi^1_i(T_1)=\pi_i^2(0)$ for every $i\in\A$. 
We define their concatenation $\Pi^2\circ\Pi^1=\{\pi^2_i\circ\pi_i^1\}_{i\in \A}$ by putting $(\pi^2_i\circ\pi_i^1)(t)=\pi_i^1(t)$ for $t\leq T_1$ and $(\pi^2_i\circ\pi_i^1)(t)=\pi_i^2(t-T_1)$ for $T_1<t\leq T_1+T_2$. 
The inversion of a feasible motion or the concatenation of concatenable feasible motions is always feasible. 

We also recall a classical concept from graph theory. Given a graph $\graph$, we say that two vertices $u,v\in\V$ are \emph{2-edge connected} if the removal of any single edge from $\G$ does not disconnect $u$ from $v$. By Menger's theorem \cite{diestel2017graph}, this condition is equivalent to saying that there are two paths with no common edge linking $u$ to $v$ (such paths are called \emph{edge-independent}). Moreover, $\G$ is \emph{2-edge connected} if any pair of vertices is 2-edge connected. 

\subsection{Existence on 2-Edge-Connected Graphs}
Grids, warehouses, and similar layouts are $2$-edge connected, the case dominating in practice; there, solvability is always guaranteed: Thm.~\ref{cor:general-solutions}'s per-agent hypothesis holds for every instance on a $2$-edge-connected graph with $\P\neq\emptyset$.
\begin{theorem}
	\label{cor:general-solutions}
	Let $\I = \left(\G, \P, \S, \A^t, \A^u, \o \right)$ be a PS-MAPF instance in which, for every tasked agent $i$, there exists a pivot $p\in\P$ such that $\o(i)$ and $p$ are $2$-edge connected.
	Then, the instance admits a solution ($Sol(\I)\neq\emptyset$).
\end{theorem}

\begin{proof}
	Consider any tasked agent $i\in\A^t$ and assume that $\o(i)$ and $p\in\P$ are $2$-edge connected. By assumption, there exist two edge independent simple paths $\gamma'=(\gamma'_0 .. \gamma'_{l'})$ and $\gamma''=(\gamma''_0 .. \gamma''_{l''})$ in $\G$ such that $\gamma'_0=\gamma''_0=\o(i)$ and $\gamma'_{l'}=\gamma''_{l''}=p$. If the two paths have no intermediate vertex in common, then the concatenation of $\gamma'$ with the reverse of $\gamma''$ yields a cycle. A concatenation of $l'$ rotations, moving $i$ along $\gamma'$, yields a motion bringing $i$ into $p$. If instead the two paths have common intermediate vertices, we consider the one of them closest to $\o(i)$ along $\gamma'$, say $v=\gamma'_{h'}=\gamma''_{h''}$ and arguing as before, we construct a motion composed of rotations bringing $i$ into $v$. A direct inductive argument based on the common intermediate points between $\gamma'$ and $\gamma''$ now permits us to construct a motion bringing $i$ to $p$ (see Figure~\ref{fig:2-edge}). Finally, notice that all agents whose starting points are within the vertices belonging to the two paths $\gamma'$ and $\gamma''$ remain $2$-edge connected to $p$ under such rotations. The remaining agents have not moved and thus remain $2$-edge connected to some pivot vertex. We can iterate this construction for each tasked agent and, by concatenating all such feasible motions, obtain a motion $\Pi$ of length $T$ by which all tasked agents pass through a pivot.
	
	We now consider the pivot-free instance $\I'=(\G, \P, \S, \emptyset, \A, \Pi(T))$. 
	By Prop.~\ref{prop:reachability}, a solution exists for $\I'$, call it $\Pi'$. 
	The concatenation $\Pi'\circ\Pi$ is finally, by construction, a successful feasible motion for $\I$.
\end{proof}

\begin{figure}
\centering
\includegraphics[height=1.4cm]{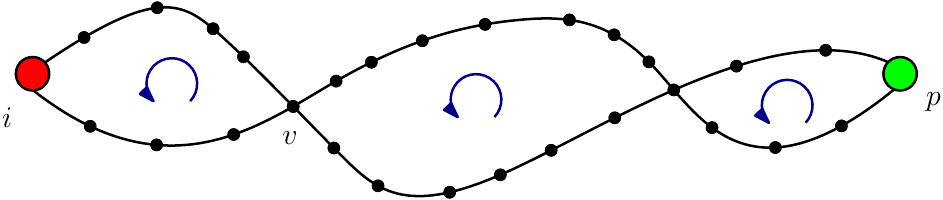}
\caption{Two edge-independent paths linking the origin of an agent (in red) to a pivot (in green). Concatenation of rotation along the obtained cycles eventually leads $i$ in $p$.}
\label{fig:2-edge}
\end{figure}

\subsection{Existence on General Graphs}
For arbitrary connected graphs, solvability admits a \emph{complete}, polynomial-time-checkable characterization based on the number and distribution of tasked and untasked agents, as well as on the topology of $\mc G$.

Given any subset of vertices $\mc W\subseteq \V$ and a configuration $\c$, we define $H(\mc W, \c) = \mc W\setminus \c(\A)$, i.e. the set of $\mc W$'s unoccupied vertices when the agents are in configuration $\c$. We also define $h(\mc W, \c)=|H(\mc W, \c)|$.

The following lemma shows that, on any connected subgraph $\G'$, feasible motions can arbitrarily move the set of unoccupied vertices in $\V'$, while keeping at rest all agents outside of it.
\begin{lemma}
	\label{lem:hole-reposition}
	Let $\I=(\G,\P,\S,\A^t,\A^u,\o)$ be a PS-MAPF instance, and let $\G'=(\V',\E')$ be a connected induced subgraph of $\G$. Let $\mc D\subseteq \V'$ be such that $|\mc D|=h(\V', \o)$. Then, there exists a feasible motion $\Pi$ of length $T$ such that $\Pi(0)=\o$, $H(\V', \Pi(T))=\mc D$, and $\pi_i(t)=\o(i)$ for every $t$ if $\o(i)\not\in \V'$.
\end{lemma}

\begin{proof}
	Consider the pivot-free sub-instance restricted to the set of agents $\A'$  currently in $\G'$, defining $\V'\setminus \mc D$ as the set of their stations: $\I'=(\G',\P,\V'\setminus \mc D,\emptyset,\A',\o|_{\A'})$. 
	Since $\G'$ is connected and, by construction, $|\A'|=|\V'|- h(\V', \o)=|\V'\setminus \mc D|$, Prop.~\ref{prop:reachability} guarantees a solution to this anonymous MAPF problem. 
	By extending the paths of all agents outside $\G'$ with wait actions, this yields the desired feasible motion on the entire graph $\G$.
\end{proof}

We now investigate motion constraints in arbitrary topologies. We notice that $2$-edge connectivity is an equivalence relation that decomposes the set of vertices $\mc V$ into disjoint subsets $C_1..C_s$, inducing subgraphs that are maximally $2$-edge connected. We call such subsets the \emph{$2$-edge connected components of $\G$}. We indicate with $\mathbb T=(\mc N, \mc F)$ the \emph{$2$-edge condensation tree of $\G$}: the tree with vertices $\mc N=\{C_1..C_s\}$ and with an edge $(C_i, C_j)$ if in $\mc G$ there exists an edge between a node in $C_i$ and a node in $C_j$ (we refer to such edges as \emph{bridges}).

In our analysis, we need to distinguish among the various $2$-edge connected components.
To this aim, we decompose the set $\mc N$ as $\mc N=\mc N_1^0\cup\mc N_1^1\cup\mc N_{>1}$. $\mc N_1^0$ and $\mc N_1^1$ indicate the set of singleton components of degree not larger than $2$ and above $2$, respectively, while $\mc N_{>1}$ indicates the set of non-singleton components (thus of cardinality at least $3$). We refer to the components in $\mc N_1^1$ as \emph{singleton branching}.
It will be useful to consider configurations in which all unoccupied vertices form a single connected region containing p; intuitively, the free space is gathered at the pivot rather than dispersed through the graph.
\begin{definition} 
	The configuration $\o$ is called $p$-\emph{adapted} if the set of unoccupied vertices $H(\V,\o)$ is either empty or it contains $p$ and the induced subgraph $\G[H(\V, \o)]$ is connected.
\end{definition}

The following result shows that the agents can always be brought to a p-adapted configuration, for any pivot $p$.
\begin{proposition}\label{lem:adapted} 
	Let $\I = \left(\G, \P, \S, \A^t, \A^u, \o \right)$ be a PS-MAPF instance and let $p\in\P$. There exists a feasible motion $\Pi$ of length $T$ such that $\Pi(T)$ is a $p$-adapted configuration.
\end{proposition}
\begin{proof}
	Consider any set of vertices $\mc W$ of cardinality $h(\V, \o)$  that contains $p$ and such that the induced subgraph is connected. Lemma \ref{lem:hole-reposition} yields a motion leading to a configuration $\o'$ such that $H(\V,\o')=\mc W$. This proves the result.
\end{proof}

We now consider a fixed (tasked) agent $i$ and a pivot vertex $p$ in $\P$, and investigate the existence of a feasible motion that brings agent $i$ from $\o(i)$ to $p$. We will give an explicit characterization of the existence, assuming the starting point is a $p$-adapted configuration. A clear sufficient condition for the existence of such a feasible motion is that $d(\o(i), p)\leq h(\V, o)$. 
 Indeed, in this case, as the original configuration is $p$-adapted, one can distribute the unoccupied vertices to follow exactly a minimal path from $\o(i)$ (excluded) to $p$ without moving $i$. Then, a direct motion, in which only agent $i$ moves along the empty path, leads it to $p$. However, this is not at all a necessary condition, as shown in the instance presented in Figure \ref{fig:deff}.

\begin{figure}
\centering
\includegraphics[height=1.2cm]{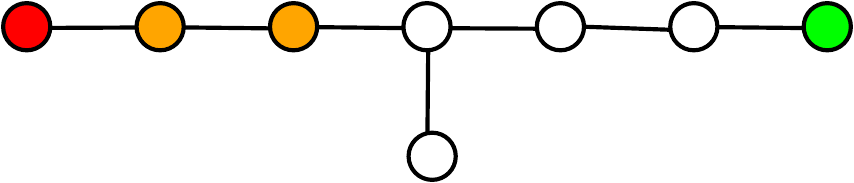}  

\vspace{7pt}\includegraphics[height=1.2cm]{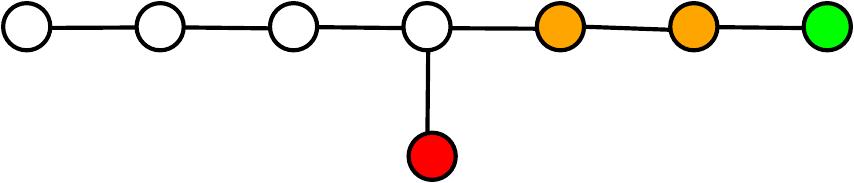}  

\vspace{7pt} \includegraphics[height=1.2cm]{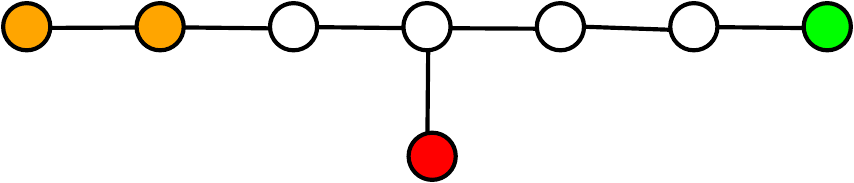}  
\caption{In the top figure, a tasked agent $i$ is positioned at the red vertex, at distance $6$ from the pivot $p$ in green. Despite there being only $4$ unoccupied vertices (in white), a feasible motion that brings $i$ to $p$ exists and can be easily constructed by starting from the initial condition and then passing through the two intermediate configurations (central and bottom figures).}
\label{fig:deff}
\end{figure}

Below, we describe a computable method for determining whether feasible motions exist that lead a tasked agent to a pivot $p$. It applies only to initial configurations that are $p$-adapted, so our characterization of the existence of solutions is not fully explicit. However, since the reduction to $p$-adapted configurations is achieved via polynomial algorithms, it is indeed fully computable.

We indicate with $C'$ and $C''$ the $2$-edge connected components to which $\o(i)$ and $p$ respectively belong. We then consider the unique path in the $2$-edge condensation tree $\mathbb T$ from $C'$ to $C''$, say $C'=C_0, C_1,\dots , C_l=C''$. Furthermore, we let $0=m_0<m_1<m_2<\cdots <m_r=l$ be such that $C_{m_h}\in \mc N_1^1\cup\mc N_{>1}$ for every $h=1,\dots , r-1$, while $C_{k}\in \mc N_1^0$ for all $k\in\{1,\dots , l\}\setminus\{m_1,\dots , m_r\}$. For $h=0, \dots, r$, the binary variable $\rho_h$ is set to $1$ exactly when $0<h<r$ and $C_{m_h}\in \mc N_1^1$, namely when we are not considering the initial or final component and the component is a singleton of degree greater than two. 
We finally define the \emph{effective distance} of agent $i$ to $p$ as the quantity
\begin{equation}\label{deff}
d^{\,\text{eff}}_p(\o(i)):=\max_{0\leq h \leq r-1} (m_{h+1}-m_h+\rho_h+\rho_{h+1})
\end{equation}
assuming, by convention, the maximum is $0$ on empty sets.
In the definition of $d_p^{\,\text{eff}}(\o(i))$, we are considering the maximum distance between two subsequent components that are not degree two singletons, augmented by one or two, depending on whether one or both components are singletons. Notice that, by construction, $d^{\,\text{eff}}_p(\o(i))\leq d(\o(i), p)$.

Figure \ref{fig:deff} (top) depicts a $p$-adapted configuration with a path of $l=6$ $2$-edge connected components from $\o(i)$ to $p$. Notice that in this case $m_1=3, m_2=6$ and $\rho_1=1$ so that 
$d^{\,\text{eff}}_p(\o(i))=\max\{3+1, 3+1\}=4$.

As the next result shows, this distance is of fundamental importance, as it determines the minimum number of unoccupied vertices required for a feasible motion that brings agent $i$ to $p$ to exist. 

\begin{lemma}
\label{lemma:solution-one-agent}
	Let $\I = \left(\G, \P, \S, \A^t, \A^u, \o \right)$ be a PS-MAPF instance and assume that $\o$ is $p$-adapted with respect to some pivot $p$. Given an agent $i\in\A^t$, there exists a feasible motion that brings $i$ to $p$ if and only if the following condition holds:
	\begin{equation}\label{cond-existence} d_p^{\,\text{eff}}(\o(i))\leq h(\V, \o)=|\V|-|\A|
	\end{equation}	
\end{lemma}
\begin{proof}
Consider the unique path in the $2$-edge condensation graph $C_0, C_1,\dots , C_l$ connecting $C_0$, the $2$-edge connected component containing $\o(i)$, to $C_l$ the $2$-edge connected component containing $p$. 
We prove both directions by induction on $l$. If $l=0$, $\o(i)$ and $p$ lie in the same $2$-edge connected component and the result follows from Thm.~\ref{cor:general-solutions} and is independent of $h(\V, \o)$. From now on, we assume that $l\geq 1$. 

We assume that condition \eqref{cond-existence} holds. 
If $|C_0|>1$ and $C_0$ contains at least one unoccupied vertex, then agent $i$, by first repositioning itself in the unique vertex $v_0\in C_0$ connecting to $C_1$ and then freely following the pattern of unoccupied vertices, can reach $p$. 
We now assume that all vertices in $C_0$ are occupied. We also assume that $\o(i)=v_0$, as it can always be reached with a motion only involving agents in $C_0$ thanks to Thm.~\ref{cor:general-solutions}, and that $d(v_0, p)>h(\V, \o)$; if not, by setting the unoccupied vertices along a minimal path from $v_0$ to $p$, we again would directly construct the wanted feasible motion.

Consider now the unique path in $\mc G$, $v_0,v_1 .. v_{m_1}$ where $v_k\in C_k$ for $k=0,1..m_1$. Removing $C_0$, consider now the set of all vertices $\mc W$ belonging to the connected component of $\mc G$ containing $p$. By construction, we know that all unoccupied vertices are in $\mc W$: $h=h(\V, \o)=h(\mc W, \o)$. 

Assume first that $\rho_1=0$. Consider any subset $\mc D$ of $\mc W$ of cardinality $h$ and possessing the property that $\{v_1 .. v_{m_1}\}\subseteq \mc D$, while the $h - m_1$ leftover vertices of $\mc D$ form a connected set of vertices containing $p$.
Assumption \eqref{cond-existence} yields $h \geq m_1$, which together with the case assumption $d(\o(i), p)>h$ ensures such a choice of $\mc D$ is possible.
It follows from Lemma \ref{lem:hole-reposition} that there exists a motion bringing the unoccupied vertices into $\mc D$ while maintaining fixed all agents $j$ with $\o(j)\not\in\mc W$, so, in particular, agent $i$. Afterward, agent $i$ can move to vertex $v_{m_1}$ along the path of empty vertices. If $r=1$, $v_{m_1}$ and $p$ are in the same $2$-edge connected component, and we conclude using Thm.~\ref{cor:general-solutions} again. Otherwise, $|C_{m_1}|>1$.  If there are unoccupied vertices within $C_{m_1}$, then there is a path of unoccupied vertices leading to $p$ and, as above, we could find a feasible motion bringing $i$ to $p$. If not, let $v^+_{m_1}\in C_{m_1}$ be the unique point in $C_{m_1}$ that is connected through an edge to some node in $C_{m_1+1}$. $2$-edge connectivity implies that there exists $v'\in C_{m_1}$ that is  connected with an edge to $v_{m_1}$ and whose removal does not disconnect $v_{m_1}$ from $v^+_{m_1}$.  A motion within the $2$-edge connected component $C_{m_1}$ leads agent $i$ to $v'$ and, afterward, unoccupied vertices can be again distributed to form a connected set containing $p$ while leaving agent $i$ still in node $v'$. We are thus now again in a $p$-adapted configuration, and the number of $2$-edge connected components between the position of $i$ and $p$ is equal to $l-m_1<l$. Induction then applies. 

If instead $\rho_1=1$ so that $C_{m_1}$ is a branching component, consider any vertex $v'\not\in\{v_{m_1-1}, v_{m_1+1}\}$ connected to $v_{m_1}$. Consider now any subset  $\mc D$ of $\mc W$ of cardinality $h(\mc W, \o)$ and possessing the property that $\{v_1 .. v_{m_1}, v'\}\subseteq \mc D$, while the remaining vertices form a connected set containing $p$. 
Assumption \eqref{cond-existence} yields $h \geq m_1+1$, which together with the case assumption $d(\o(i), p)>h$ ensures such a choice of $\mc D$ is possible.
t follows from Lemma \ref{lem:hole-reposition} that there exists a motion bringing the unoccupied vertices into $\mc D$ while maintaining fixed all agents $j$ with $\o(j)\not\in\mc W$, so, in particular, agent $i$. Afterward, agent $i$ can move to vertex $v'$ along the path of empty vertices. If we now remove vertex $v'$, we notice that, by construction, all empty vertices are in the component containing $p$ (in the line $\{v_0 .. v_{m_1}\}$ and in a connected subset containing $p$). This implies that there exists a feasible motion leading again to a $p$-adapted configuration without moving agent $i$ currently in $v'$.
Now, the new distance in the condensation tree between the components containing $i$ and $p$ is equal to $l-m_1+1$. If $m_1>1$, this is strictly less than $l$. In the case when $m_1=1$, we cannot directly apply the inductive reasoning. However, notice that, in this case, $d_p^{\,\text{eff}}(\o(i))\geq m_2-m_1+\rho_2+\rho_1=m_2+\rho_2$. This implies that, in this case, we can work directly with the path from $\o(i)$ to $C_{m_2}$, omitting the intermediate step $C_{m_1}$. In both cases $\rho_2=0$ and $\rho_2=1$, previous considerations allow us to apply the inductive argument and thus conclude the proof.

Conversely, suppose that there exists a feasible motion $\Pi=\{\pi_j\}_{j\in\mc A}$ that brings $i$ to $p$. Necessarily, agent $i$ will need to touch all $ 2$-edge-connected components along the path $C_0, C_1,\dots, C_l$ during its motion. Let $v_0\in C_0$ be the only node that is connected to $C_1$ through an edge. Consider then the unique path in $\mc G$, $v_0,v_1 .. v_{m_1}$ where $v_k\in C_k$ for $k=0,1..m_1$. Let $t_1$ be the first time $i$ visits $v_{m_1}$ and let  $t_0$ be the last time below $t_1$ that $i$ visits $v_{0}$. Then, at every instant $t\in [t_0, t_1]$, agent $i$ will be in one of the vertices $v_{k(t)}$ for $k(t)\in 0.. m_1$. 

Indicate with $h_0\leq h(\mc V, \o)$ the number of unoccupied vertices at time $t_0$ present in the connected component containing $p$ when $v_0$ is removed. 
We indicate with $h(t)$ the number of unoccupied vertices at time $t$ present in the connected component containing $p$ when the node $v_{k(t)}$ containing $i$ is removed. A direct inductive reasoning shows that $h(t)=h_0-k(t)$. In particular, $0\leq h(t_1)=h_0-m_1$ that yields $h_0\geq m_1$. In case $\rho_1=1$ and $h_0=m_1$, there is no unoccupied position in the connected component containing $p$ when the node $v_{m_1}$ is removed, and thus, since all edges connecting $v_{m_1}$ to nodes different from $v_{m_1-1}$ are bridges, no further motion is possible. We assume that $\Pi$ is a solution, so eventually $i$ must move further, contradicting the assumption made. This implies that $h_0\geq m_1+1$. In any case, this proves that 
\begin{equation}\label{cond1}h(\mc V, \o)\geq  m_1+\rho_1\end{equation}
As in the direct implication, in the case $m_1=1$ and $\rho_1=1$, we can not directly apply the inductive argument on distance, as when $i$ reaches $v'$, it is at the same distance from $p$ as it was in $v_0$. In this case, we argue that also the condition 
\begin{equation}\label{cond2}h(\mc V, \o)\geq m_2+\rho_2\end{equation} 
holds. This can be seen as follows. Let $s_2$ be the first time $i$ visits $v_{m_2}$ and let  $s_1$ be the last time below $s_2$ that $i$ visits $v_{1}=v_{m_1}$. Consider, moreover, $s_1^-=\max\{t\leq s_1\,|\, \pi(t)\neq v_{m_1}\}$. By construction, we have that, at every instant $t\in [s_1^-+1, s_2]$, agent $i$ will be in one of the vertices $v_{k(t)}$ for $k(t)\in m_1.. m_2$. Indicate now with $h_1$ the number of unoccupied vertices at time $s_1^-+1$ present in the connected component containing $p$ when $v_{m_1}$ is removed. Notice that the motion of $i$ at time $s_1^-$ has left at least one empty node in one of the components not containing $p$. This yields $h(\mc V, \o)\geq h_1+1$. As above, we indicate with $h(t)$ the number of unoccupied vertices at time $t$ present in the connected component containing $p$ when the node $v_{k(t)}$ containing $i$ is removed. Direct inductive reasoning shows that $h(t)=h_1-k(t)+1$. In particular, $0\leq h(s_2)=h_1-m_2+1$ that yields $h_1\geq m_2-1$. With the same argument as above, we can improve the estimation to $h_1\geq m_2-1+\rho_2$ that, together with $h(\mc V, \o)\geq h_1+1$, yields \eqref{cond2}.

In both cases (either \eqref{cond1} or, when $m_1=\rho_1=1$, \eqref{cond2}), using the direct implication (i.e. that condition \eqref{cond-existence} is sufficient for existence), we can bring $i$, through a feasible motion $\Pi'$ of length $T'$, into a new position $\o'(i)=\pi'_i(T)$ whose distance in the condensation graph to $p$ is strictly reduced, and the configuration $\o'=\pi'(T)$ is still $p$-adapted. Notice that from $\o'$ there exists a feasible motion bringing $i$ to $p$: it is sufficient to consider the concatenation of the inversion of $\Pi'$ with $\Pi$, namely the feasible motion $\Pi\circ(\Pi')^{-}$. By induction, $d^{\,\text{eff}}_p(\o'(i))\leq h(\V, \o')=h(\V, \o)$  and since, by the way $\o'$ was defined,
$$d^{\,\text{eff}}_p(\o'(i))=\max_{1\leq h \leq r-1} (m_{h+1}-m_h+\rho_h+\rho_{h+1})$$
combined with $h(\V, \o)\geq h_0\geq m_1+\rho_1$ yields the validity of relation \eqref{cond-existence}.
\end{proof}

The effective distance defined in \eqref{deff} depends on the particular $p$-adapted configuration we choose. However, the statement itself of Lemma \ref{lemma:solution-one-agent} implies that condition \eqref{cond-existence} is a characterization of existence and thus it does not depend on the particular $p$-adapted configuration.

Lemma \ref{lemma:solution-one-agent} yields the following result. 

\begin{theorem}
\label{theorem:solution-agents}
A PS-MAPF instance $\I = \left(\G, \P, \S, \A^t, \A^u, \o \right)$ is solvable if and only if, for every tasked agent $i \in \A^t$, there exist a pivot $p \in \P$ and a $p$-adapted configuration $\c$, reachable from $\o$ by a feasible motion, such that
\begin{equation}\label{theo-existence}d^{\,\text{eff}}_{p}(\c(i)) \leq h(\V, \o) = |\V|-|\A|.\end{equation}
\end{theorem}

\begin{proof}
'If': Proposition \ref{lem:adapted} and Lemma \ref{lemma:solution-one-agent} (using condition \eqref{theo-existence}) guarantee, for every tasked agent $i$, the existence of a feasible motion $\Pi^i$ bringing agent $i$ to touch a pivot. The concatenations of all such motions with their inversions (i.e., for agent $i$, we create a motion that reaches $\c(i)$, visits $p$, and inverts back to $\o(i)$, and then we concatenate these motions across agents) yield a feasible motion $\Pi$ allowing all tasked agents to touch a pivot. We now consider the pivot-free instance $\I'=(\G, \P, \S, \emptyset, \A, \Pi(T))$. 
By Prop.~\ref{prop:reachability}, a solution $\Pi'$ exists for $\I'$.
The concatenation $\Pi'\circ\Pi$ is finally, by construction, a successful feasible motion for $\I$.
	 
'Only if': Suppose $\Pi=\{\pi_i\}_{i\in\mc A}$ is a solution of length $T$. For every agent $i\in\mc A^t$, let $p$ be a pivot such that $\pi_i(t_i)=p$ for some $t_i\leq T$. Consider the restriction $\Pi_{|[0, t_i]}$ of $\Pi$ to the time interval $[0, t_i]$. Let $\tilde{\Pi}^i$ be a motion of length $T^i$  that brings the system, originally in configuration $\o$, into a $p$-adapted configuration $\c$. Then, $\Pi_{|[0, t_i]}\circ\tilde{\Pi}^{i-}$ is a feasible motion from configuration $\c$ that leads $i$ to pivot $p$. Thanks to Lemma \ref{lemma:solution-one-agent}, condition \eqref{theo-existence} holds true.
\end{proof}

Algorithmically, the check of necessary conditions in Thm.~\ref{theorem:solution-agents} can proceed as follows. For every pivot $p\in\mc P$, we bring the system into \emph{any} $p$-adapted configuration, and we let $\mc A^t(p)$ be the set of tasked agents for which condition \eqref{cond-existence} is satisfied. Existence of a solution is then equivalent to $\cup_{p\in\mc P}\mc A^t(p)=\mc A^t$. The complexity is bounded, in the worst case, by $|\mc P|$ polynomial reductions to pivot adapted configurations, and for each of them, the check of $|\mc A^t|$ conditions as \eqref{theo-existence}.

Finally, a simple structural condition guarantees solvability and underlies our experimental instances.

\begin{definition}
	\label{def:well-formed}
	A PS-MAPF instance $\I = (\G, \P, \S, \A^t, \A^u, \o)$ is \emph{well-formed} if, for every tasked agent $i\in\A^t$, there is a path from $\o(i)$ to some pivot $p\in\P$ none of whose vertices, other than the start $\o(i)$, is the origin $\o(j)$ of another agent $j\neq i$.
\end{definition}

\begin{proposition}
\label{prop:well-formed}
Every well-formed instance is solvable.
\end{proposition}

\begin{proof}
    Let $\I = (\G, \P, \S, \A^t, \A^u, \o)$ be a well-formed instance and consider a tasked agent $i\in\A^t$. Since $\I$ is well-formed, there exists a pivot $p\in\P$ and a path $\gamma$ from $\o(i)$ to $p$ such that all vertices along $\gamma$, with the exception of $\o(i)$, are unoccupied by any other agent. 

    Consider the subgraph $\G'$ induced by the vertices of $\gamma$, $\V'$, and the corresponding restricted sub-instance for agent $i$: $\I = (\G', \{p\}, \V', i, \emptyset, \o(i))$. In this restricted setting, the initial configuration is trivially $p$-adapted and clearly satisfies condition \eqref{cond-existence}. Thus, we can obtain a feasible motion $\Pi^i_i$ that routes agent $i$ from $\o(i)$ to $p$ within $\G'$. 

    We can then construct a feasible motion $\Pi^i$ for the full instance $\I$ by applying the movements of $\Pi^i_i$ to agent $i$ while assigning wait actions to all other agents. This motion is strictly feasible because only agent $i$ moves (precluding swapping conflicts), and it does so exclusively through unoccupied vertices (precluding vertex conflicts).

    By concatenating $\Pi^i$ with its inversion ${\Pi^i}^-$, agent $i$ visits $p$ and the system subsequently returns to the initial configuration $\o$. Iterating this construction sequentially across all tasked agents yields a concatenated feasible motion $\Pi$ of length $T$ that allows every tasked agent to touch a pivot. 

    We now consider the pivot-free instance $\I'=(\G, \P, \S, \emptyset, \A, \Pi(T))$. By Proposition~\ref{prop:reachability}, a solution $\Pi'$ exists for $\I'$. The concatenation $\Pi'\circ\Pi$ is finally, by construction, a successful feasible motion for $\I$.
\end{proof}

%% file: Sections/complexity.tex
\section{Complexity Analysis}
\label{sec:complexity}
We prove the \NP-hardness of minimum station-makespan (MSM) and station-flowtime (MSF) decision problems.

\begin{problem}
    Given a PS-MAPF instance $\I$ and $k\in\mathbb{N}$, is there a solution $\Pi$ with $F_{\S}^{\I}(\Pi)\leq k$, where $F\in\{M,\Sigma\}$? We refer to the \emph{MSM decision problem} when $F=M$, and to the \emph{MSF decision problem} when $F=\Sigma$.
\end{problem}

We reduce from 3SAT \cite{cook1971complexity} in a similar fashion to Yu and LaValle \shortcite{yu2013structure}. Our key additional insight is that strategically positioning the pivot in our constructed instance forces every agent to a predetermined station, thereby making the analysis similar to classical MAPF.

Let $(X,C)$ be a 3SAT instance with $n$ variables and $m$ clauses, in conjunctive normal form \cite{sipser2012introduction} with at most three literals per clause. For each variable $x_i$, create vertices $v_{x_i}$ and $v_{x_i}'$ joined by two internally disjoint paths of length $m+3$, the ``upper'' and the ``lower'' path (Fig.~\ref{fig:station-construction}); for each clause $c_j$, create vertices $v_{c_j}$ and $v_{c_j}'$, connecting the $v_{c_j}'$ as a path from $v_{c_1}'$ to $v_{c_m}'$; create a vertex $p$ adjacent to $v_{c_m}'$ and to all the $v_{x_i}$. Finally, for each $c_j$ and each literal $x_i\in c_j$, connect $v_{c_j}$ to the $j^{th}$ vertex from $v_{x_i}$ on the upper path if $x_i$ is unnegated in $c_j$, and on the lower path if negated. The stations are $\S=\S^c\cup \S^x$ with $\S^c=\{v_{c_j}'\,|\, j\in[1..m]\}$, and $\S^x=\{v_{x_i}'\,|\, i\in[1..n]\}$; the agents are $m$ tasked ``clause agents'' $a_{c_j}$ with $\o(a_{c_j})=v_{c_j}$ and $n$ untasked ``variable agents'' $a_{x_i}$ with $\o(a_{x_i})=v_{x_i}$, yielding the instance $\I=(\G,\{p\},\S, \A^t,\A^u,\o)$.

\begin{figure}[tbp]
	\centering
	\includegraphics[width=0.78\columnwidth]{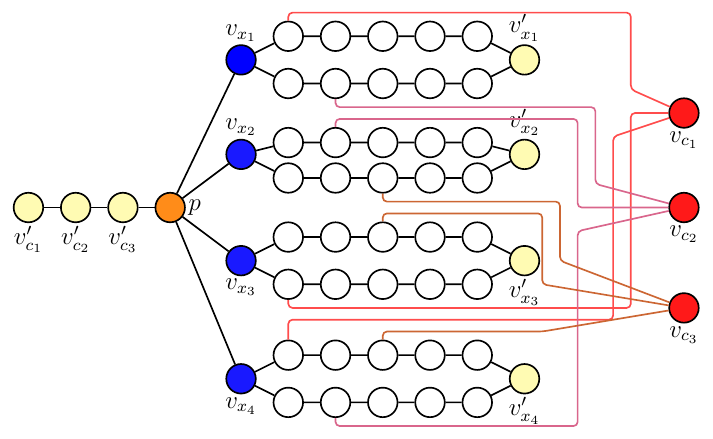}
	\caption{Construction of PS-MAPF instance $\I$ from 3SAT formula $(x_1\vee\overline{x_3}\vee x_4)\wedge(\overline{x_1}\vee x_2\vee\overline{x_4})\wedge(\overline{x_2}\vee x_3\vee x_4)$. Red/blue vertices are tasked/untasked agent origins, respectively. Yellow vertices are stations; orange vertices are pivots.}
	\label{fig:station-construction}
\end{figure}

\begin{proposition}\label{prop:MSMA}
    If $(X,C)\in\text{3SAT}$, then there is a solution $\Pi$ with station-makespan $M_{\S}^{\I}(\Pi)$ at most $m+3$.
\end{proposition}

\begin{proof}
    For each $x_i\in X$, if $x_i$ is set to true in the satisfying assignment, let $a_{x_i}$ use the lower path to $v_{x_i}'$. If set to false, let it use the upper path. Because this is a satisfying assignment, each $c_j\in C$ contains at least one true literal, say corresponding to the variable $x_i$. If $x_i$ is unnegated in $c_j$, then $x_i$ is set to true, and we have routed $a_{x_i}$ on the lower path. By our construction, $v_{c_j}$ connects to the upper path, which is now guaranteed to be unobstructed. Alternatively, if $x_i$ is negated in $c_j$, then $x_i$ is false and thus $a_{x_i}$ uses the upper path. In this case, we would have connected $v_{c_j}$ to the lower path, which is again unobstructed for $a_{c_j}$. In short, each $a_{c_j}$ will have an unobstructed path to a $v_{x_i}$.
            
    Furthermore, because each clause origin connect to upper/lower paths at a unique index, the clause agents arrive to the $v_{x_i}$ vertices at staggered times, preventing vertex conflicts. Specifically, $a_{c_1}$ will arrive first, proceeding to $p\in V$ and then to $v_{c_1}'$, followed sequentially by the remaining clause agents, which will filter into their respective stations. We have now described a solution $\Pi$ where all agents traverse $m+3$ edges without stopping, thus $M_{\S}^{\I}(\Pi)\le m+3$.
\end{proof}
	
To prove the converse statement, we first establish several preliminary results.

\begin{lemma}\label{lemma:dist} 
    The graph $\mc G$ has the following properties:
\begin{enumerate}
    \item[(i)] $d(v_{x_i},v_{x_i}')=m+3$ for every $i$ in $[1..n]$;
    \item[(ii)] $d(v_{x_i},v_{x_\ell}')\geq m+5$ for every $i,\ell $ in $[1..n]$ with $i\neq \ell$;
    \item[(iii)] $d(p,v_{x_i}')=m+4$ for every $i$ in $[1..n]$;
    \item[(iv)] $d(p,v_{c_j}')=m-j+1$ for every $j$ in $[1..m]$;
    \item[(v)] $d(p,v_{c_j})=j+2$ for every $j$ in $[1..m]$;
\end{enumerate}
\end{lemma}
\begin{proof} 
    Direct consequences of the way $\G$ is defined.
\end{proof}
		
A solution $\Pi= \{\pi_i\}_{i\in \A}$ is called \emph{non-crossing} if $\pi_a(T)\in\S^c$ when $a\in\A^t$ and $\pi_a(T)\in\S^x$ when $a\in\A^u$. The following result holds.

\begin{lemma}\label{lemma:noncross} 
    Suppose $\Pi= \{\pi_i\}_{i\in \A}$ is a solution of $\I$ such that $M_{\S}^{\I}(\Pi)\leq m+3$. Then, $\Pi$ is non-crossing.
\end{lemma}
\begin{proof} 
    Suppose to the contrary that there is a $j\in[1..m]$ such that $\pi_{a_{c_j}}(T)\in \S^x$. Since clause agents must pass through $p$, items (iii) and (v) of Lemma \ref{lemma:dist} imply that 
    \[t_{\S}(\pi_{a_{c_j}})\geq m+6+j>m+3,\]
    contradicting the assumption that $M_{\S}^{\I}(\Pi)\leq m+3$.
\end{proof}
        
\begin{lemma}\label{lemma:noncross2} 
    Suppose $\Pi= \{\pi_i\}_{i\in \A}$ is a solution for $\I$ of length $T$ such that $M_{\S}^{\I}(\Pi)\leq m+3$. Then, 
    \[\pi_{a_{x_i}}(T)=v'_{x_i}\ \forall i\in[1..n],\quad \pi_{a_{c_j}}(T)=v'_{c_j}\ \forall j\in[1..m].\]
\end{lemma}
\begin{proof} 
    Considering the variable agents, Lemma \ref{lemma:noncross} tells us that $\Pi$ is non-crossing and, because of (i) and (ii) of Lemma \ref{lemma:dist}, we must have that $\pi_{a_{x_i}}(T)=v'_{x_i}$ for every $i\in[1..n]$. Regarding the clause agents, let $\Omega=\{j\,|\, \pi_{a_{c_j}}(T)\neq v'_{c_j}\}$. If $\Omega$ is empty, the proof is complete. Otherwise, for $\bar h=\max \Omega$, we must have $\pi_{a_{c_{\bar h}}}(T)= v'_{c_h}$ with some $h<\bar h$. It then follows from items (iv) and (v) of Lemma \ref{lemma:dist} that
    \[t_{\S}(\pi_{a_{c_{\bar h}}})\geq m+3+\bar{h}-h>m+3,\]
    a contradiction.
 \end{proof}

\begin{proposition}\label{prop:MSMB} 
    If there is a solution $\Pi$ with $M_{\S}^{\I}(\Pi)\le m+3$, then $(X,C)\in\text{3SAT}$.
\end{proposition}
\begin{proof} 
    It follows from Lemmas \ref{lemma:dist} and \ref{lemma:noncross2} that each agent $a$ follows a shortest path $\pi_a$ to its station destination without wait actions. In particular, since upper and lower paths are internally disjoint, a shortest path between $v_{x_i}$ and $v_{x_i}'$ is one of these two paths exclusively. Consequently, the variable agent $a_{x_i}$ occupies only one of these paths to reach $v_{x_i}'$, ensuring the other path is free of variable agents and available for clause agents to move left on unobstructed.
    
    We extract a satisfying assignment for $(X,C)$ based on the paths that the clause agents take. Suppose $a_{c_j}$ uses the upper path of $v_{x_i}$. By construction, $v_{c_j}$ connects to the upper path of $v_{x_i}$ only if $x_i$ is unnegated in $c_j$. So, in this case, we set $x_i$ to true and thus $c_j$ is satisfied. Alternatively, if $a_{c_j}$ uses the lower path, then $x_i$ must be negated in $c_j$, so we set $x_i$ to false. Each clause agent must go through some unobstructed upper/lower path, so we can use this process to satisfy each clause. Thus, $(X,C)$ has a satisfying assignment.
\end{proof}

\begin{theorem}\label{thm:msm-hardness} 
    The MSM decision problem is \NP-hard.
\end{theorem}
\begin{proof} 
    Direct consequence of Prop. \ref{prop:MSMA} and \ref{prop:MSMB}.
\end{proof}

For the MSF decision problem, we show that $\Sigma_\S^\I(\Pi)\leq (m+n)(m+3)$ is equivalent to $M_\S^\I(\Pi)\leq m+3$ on the constructed instance.

\begin{lemma}\label{lemma:cross} 
    Given a solution $\Pi= \{\pi_i\}_{i\in \A}$ of the instance $\I$ of length $T$, the following facts hold.
    \begin{enumerate}
        \item[(i)] if an agent $\bar a$ is such that $\pi_{\bar a}(T)\in\S^c$, then $t_{\S}(\pi_{\bar a})\geq m+1$.
        \item[(ii)] if $\Pi$ is non-crossing, $t_{\S}(\pi_{a})\geq m+3$ for every $a\in\A$.
    \end{enumerate}
\end{lemma}
\begin{proof} 
    Suppose $\bar a$ is an agent such that $\pi_{\bar a}(T)=v'_{c_j}$ for some $j\in[1..m]$. Let $\A'$ be the set of agents (tasked or untasked) $a'$ such that $\pi_{a'}(T)\in \{v'_{c_1},\dots ,v'_{c_{j-1}}\}$. As every path from origins to any clause-stations passes through $p$, every agent $a\in\A'\cup\{\bar a\}$ must necessarily pass through $p$ in their motion $\pi_a$. Moreover, if we let $\tau_a=\max\{t\in[0..T]\,|\, \pi_a(t)=p\}$ to be the last time that agent $a$ visits the pivot node $p$, we must have 
    \begin{equation}\label{ineq-line}
        \tau_{\bar a}>\tau_{a'}\quad\forall a'\in\A'
    \end{equation}
    This can be seen formally as follows. Consider the natural total order on the line graph induced by the vertices $\S^c\cup\{p\}$, increasing from left to right. Then, $\pi_a(t)\in\S^c$ and $\pi_a(t)<p$ for every $t>\tau_a$ for every $a$ in $\A'\cup\{\bar a\}$. 
    If $\tau_{\bar a}<\tau_{a'}$ for some $ a'\in\A'$, then, both agents $\bar a$ and $a'$ are in $\S^c\cup\{p\}$ from $\tau_{a'}$ on. Since, by construction, $\pi_{\bar a}(\tau_{a'})<\pi_{a'}(\tau_{a'})=p$, the two agents must maintain such order of position at any future times (as we are assuming the absence of vertex or edge conflicts). However, this contradicts the fact that $\pi_{a'}(T)<v'_{c_j}=\pi_{\bar a}(T)$.
    
    As the absence of conflicts implies that $a\mapsto \tau_a$ is injective among the agents passing through $p$, inequality \eqref{ineq-line} yields $\tau_{\bar a}\geq j$. Finally,
    \[t_{\S}(\pi_{\bar a})\geq \tau_{\bar a}+d(p, v'_{c_j})\geq j+m-j+1=m+1,\]
    proving (i). If $\Pi$ is non-crossing, in the above considerations, we have that $\A'\subseteq \A^t$ and the stronger inequality $\tau_{a'}\geq 3$ for every $a'\in\A'$. The estimation above now reads
    \[t_{\S}(\pi_{\bar a})\geq \tau_{\bar a}+d(p, v'_{c_j})\geq j+2+m-j+1=m+3.\]
    Finally if $a\in\A^u$, $t_{\S}(\pi_{a})\geq m+3$ because of estimations (i) and (ii) in Lemma \ref{lemma:dist}. This proves (ii).
\end{proof}
    
\begin{lemma} \label{lem:sf-exact-time}
    Given any solution $\Pi= \{\pi_i\}_{i\in \A}$ of the constructed instance $\I$, if $\Sigma_\S^\I(\Pi)\leq(m+n)(m+3)$, then $\Pi$ is non-crossing.
\end{lemma}
\begin{proof} 
    Let
    $\A^u(\Pi)=\{a\in \A^u\,|\, \pi_a(T)\in \S^c\}$,  
    $\A^t(\Pi)=\{a\in \A^t\,|\, \pi_a(T)\in \S^x\}$, $\A^u(\Pi)^c=\A^u\setminus \A^u(\Pi)$, and  $\A^t(\Pi)^c=\A^t\setminus \A^t(\Pi)$. Since $|\A|=|\S|$, necessarily, $q=|\A^u(\Pi)|=|\A^t(\Pi)|$. For every tasked agent $a_{c_j}\in  A^t(\Pi)^c$, we let $\ell_j\in [1..m]$ to be the index such that $\pi_{a_{c_j}}(T)=v'_{c_{\ell_j}}$. We now estimate the station-flowtime as follows:
    \begin{equation}
    \begin{array}{rcl}
        \Sigma_\S^\I(\Pi)&=&\sum\limits_{a\in\A}t_{\S}(\pi_a)\\[5pt]
        &=& \sum\limits_{a\in\A^u(\Pi)}t_{\S}(\pi_a)+\sum\limits_{a\in\A^u(\Pi)^c}t_{\S}(\pi_a)\\
        &+&\sum\limits_{a\in\A^t(\Pi)}t_{\S}(\pi_a)+\sum\limits_{a\in\A^t(\Pi)^c}t_{\S}(\pi_a)\\[5pt]
        &\geq& q(m+1)+(n-q)(m+3)\\[2pt]
        &+&\sum\limits_{a_{c_j}\in\A^t(\Pi)}(j+2+m+4)\\
        &+&\sum\limits_{a_{c_j}\in\A^t(\Pi)^c}(m+3+j-\ell_j)\\
        &=& q(m+1)+(n-q)(m+3)\\
        &+&(m+6)q+(m+3)(m-q) \\
        &+&\sum\limits_{j=1}^mj-\sum\limits_{a_{c_j}\in\A^t(\Pi)^c}\ell_j\\
        &\geq& (m+n)(m+3)+q,
    \end{array}
    \end{equation}
    where we have used Lemmas \ref{lemma:dist} and \ref{lemma:cross} in the first inequality, and the fact that $j\mapsto \ell_j$ is an injective map in the second inequality. Thus, the assumption implies that $q=0$, telling us  $\Pi$ is non-crossing.
\end{proof}

\begin{proposition}\label{prop:MSF} 
    Given any solution $\Pi= \{\pi_i\}_{i\in \A}$ of the constructed instance $\I$, $\Sigma_\S^\I(\Pi)\leq (m+n)(m+3)$ if and only if  $M_\S^\I(\Pi)\leq m+3$.
\end{proposition}
\begin{proof}
    We only need to show the ``only if'' part.  Assume that $\Sigma_\S^\I(\Pi)\leq (m+n)(m+3)$. We know from Lemma \ref{lem:sf-exact-time} that $\Pi$ is non-crossing. Then, Lemma \ref{lemma:cross} implies that $t_\S(\pi_a)\geq m+3$ for every agent $a$. Consequently, $t_\S(\pi_a)= m+3$ for every $a\in\A$ and thus $M_\S^\I(\Pi)= m+3$.
 \end{proof}
  
\begin{theorem}\label{thm:msf-hardness}
	The MSF decision problem is \NP-hard.
\end{theorem}
\begin{proof} 
    Consequence of Theorem \ref{thm:msm-hardness} and Proposition \ref{prop:MSF}.
\end{proof}

Thms.~\ref{thm:msm-hardness} and~\ref{thm:msf-hardness} highlight a sharp complexity jump: pivot-free PS-MAPF is anonymous MAPF, hence polynomial-time solvable, yet our reduction uses a single pivot, so combining anonymous destinations with even one pivot-visit requirement makes both problems intractable.

%% file: Sections/algorithms.tex
\section{Algorithms}
\label{sec:algorithm}
Building on Secs.~\ref{sec:existence} and~\ref{sec:complexity}, we develop three algorithms: a \emph{complete baseline} (BA, Sec.~\ref{sec:baseline}) derived from the constructive solvability analysis, which solves every solvable instance but with very poor quality; a \emph{SAT-based solver} (Sec.~\ref{subsec:optimal}) computing makespan-optimal solutions at worst-case exponential cost; and, between the two, polynomial-time \emph{Pivot-Prioritized Planning} (PPP, Sec.~\ref{sec:ppp}), of far higher quality than BA at the price of completeness: it may fail on solvable instances, but any solution it returns is correct.

\subsection{Baseline Algorithm (BA)}
\label{sec:baseline}
Our solvability analysis in Section~\ref{sec:existence} is constructive and directly yields a complete algorithm, which we adopt as a baseline and call \emph{BA}. The algorithm processes tasked agents one at a time: for each tasked agent, it first brings the fleet into a $p$-adapted configuration (Proposition~\ref{lem:adapted}) and then routes the agent to a pivot via the motion constructed in the proof of Lemma~\ref{lemma:solution-one-agent}. Once all tasked agents have passed through a pivot, the residual instance is pivot-free and is solved as an anonymous MAPF problem (Proposition~\ref{prop:reachability}). Concatenating these feasible motions yields a solution whenever one exists (Theorem~\ref{theorem:solution-agents}), making the algorithm complete. BA, however, prioritizes completeness over solution quality: its sequential, one-agent-at-a-time construction yields station-makespan and station-flowtime values that leave substantial room for improvement.

\begin{algorithmSM}[htb!]
	\caption{Baseline Algorithm (BA)}
	\label{alg:baseline-algorithm}
	\begin{algorithmic}[1]
		\renewcommand{\algorithmicrequire}{\textbf{Input:}}
		\renewcommand{\algorithmicensure}{\textbf{Output:}}
		\REQUIRE Feasible Instance $\I=(\G,\P,\S,\A^t,\A^u,\o)$
		\ENSURE Global Solution $\Pi$
		
		\STATE $\Pi \leftarrow \emptyset$, $\sigma \leftarrow \o$
		\STATE $\{\mc N, \mc F\} \leftarrow \textsc{2EdgeCondensationTree}(\G)$
		\STATE $\{\mc N_1^0, \mc N_1^1, \mc N_{>1}\} \leftarrow \textsc{Decompose}(\mc N)$
		
		\FORALL{$a \in \A^t$}
		\STATE $\Pi' \leftarrow \emptyset$
		
		\STATE $p, \pi_{setup} \leftarrow \textsc{SelectActivatablePivot}(a, \sigma, \P, \G)$
		\STATE $\textsc{ApplyMotion}(\pi_{setup}, \sigma)$
		\STATE append $\pi_{setup}$ to $\Pi'$
		
		\STATE $P \leftarrow \textsc{ShortestPath}(\sigma(a), p, (\mc N, \mc F))$
		\STATE $\{m_0, \dots, m_r\} \leftarrow \textsc{DecomposePath}(P, \mc N_1^1 \cup \mc N_{>1})$
		
		\FORALL{$h=0$ to $r-1$}
		\STATE $u \leftarrow m_h$, $v \leftarrow m_{h+1}$
		\IF{$u \in \mc{N}_{>1}$}
		\STATE $\pi_{clear} \leftarrow \textsc{Clear2EC}(u, v, \G, \sigma, a)$
		\ELSE 
		\STATE $\pi_{clear} \leftarrow \textsc{ClearBranching}(u, v, \G, \sigma, a)$
		\ENDIF
		\STATE $\textsc{ApplyMotion}(\pi_{clear}, \sigma)$
		\STATE append $\pi_{clear}$ to $\Pi'$
 
		\STATE $\pi_{adv} \leftarrow \textsc{Advance}(u, v, \G, \sigma, a)$
		\STATE $\textsc{ApplyMotion}(\pi_{adv}, \sigma)$
		\STATE append $\pi_{adv}$ to $\Pi'$
		\ENDFOR
		
		\IF{$m_r \in \mc N_{>1}$}
		\STATE $\pi_{adv} \leftarrow \textsc{Advance}(m_r, p, \G, \sigma, a)$
		\STATE $\textsc{ApplyMotion}(\pi_{adv}, \sigma)$
		\STATE append $\pi_{adv}$ to $\Pi'$
		\ENDIF
		
		\STATE $\Pi'^- \leftarrow \textsc{InvertMotion}(\Pi')$
		\STATE $\textsc{ApplyMotion}(\Pi'^-, \sigma)$
		\STATE append $\Pi'$ followed by $\Pi'^-$ to $\Pi$
		
		\ENDFOR
		
		\STATE $\pi_{final} \leftarrow \textsc{ARouting}(\G, \sigma, \S)$
		\STATE append $\pi_{final}$ to $\Pi$
		
		\STATE \textbf{return} $\Pi$
	\end{algorithmic}
\end{algorithmSM}

In \emph{Phase 1}, tasked agents sequentially compute and execute paths to a pivot, subsequently rewinding the system state to ensure independent solvability (lines 4-32). The procedure relies on the graph's 2-edge condensation tree, which is computed (line 2) and partitioned into distinct component types (line 3): non-singletons, branching singletons, and degree-two singletons. For a given tasked agent $a \in \A^t$, the algorithm identifies an activatable pivot $p$ (line 6) and determines the shortest path of components in the condensation tree leading to it (lines 9-10). The core of the pivot-routing logic (lines 11-23) iteratively prepares the graph for the agent's progression: it repositions the unoccupied vertices (holes) to "clear" the way to the next component along the sequence. This is achieved using specific subroutines depending on whether the current component is a non-singleton 2-edge component or a branching singleton (lines 13-17), before advancing the agent into it (lines 19-21). A final advance step is performed if the pivot lies within a non-singleton component (lines 24-28). Crucially, to prevent agents from blocking each other, the algorithm inverts the agent's entire motion sequence $\Pi'$ into $\Pi'^-$ (line 29). Applying this inverse motion (line 30) completely rewinds the global state $\sigma$ back to the initial configuration $\o$. The concatenated forward and reverse motions are appended to the global solution (line 31), guaranteeing that every subsequent agent operates under the exact same initial conditions without interference.

In \emph{Phase 2}, after all tasked agents have visited their designated pivots and the global state has been safely restored to $\sigma = \o$, the problem reduces to a pivot-free instance. The algorithm executes an anonymous MAPF routing (line 33) to direct the entire fleet to the final station vertices in $\S$. Because all pivot requirements are already satisfied in Phase 1, this final step leverages classical flow-based techniques to move the agents to their destinations without tracking their identities. The resulting motion $\pi_{final}$ is appended to complete the global solution $\Pi$ (lines 34-36).

\textbf{Complexity Analysis.} Let $n=|\V|$, $m=|\E|$. Computing the 2-edge condensation tree and its decomposition (lines 2-3) takes $\mc{O}(n + m)$ time. In Phase 1, the outer loop processes $|\A^t| \le n$ tasked agents. For each agent, \textsc{SelectActivatablePivot} (line 6) evaluates at most $|\P| \le n$ pivots. For each candidate, verifying condition \eqref{cond-existence} requires reducing the system to a $p$-adapted configuration. Because this sub-problem is \emph{anonymous}, we only need to route at most $n$ unoccupied vertices (holes). Repositioning these holes along a spanning tree requires at most $\mc{O}(n^2)$ operations \cite{yu2013multiagent}. Crucially, because the global state is rewound to $\o$ after each agent, this evaluation can be computed once offline for all pivots and cached, taking $\mc{O}(|\P| n^2)$ time globally across the entire algorithm. Determining the shortest path in the condensation tree (line 9) takes $\mc{O}(n)$. The inner loop (lines 11-23) advances the agent across at most $n$ components. At each step, repositioning the holes (\textsc{Clear2EC}, \textsc{ClearBranching}) reduces to an anonymous MAPF problem costing $\mc{O}(n^2)$ operations \cite{yu2013multiagent}. Advancing the agent through a 2-edge-connected component (as established in Theorem \ref{cor:general-solutions}) requires finding two edge-independent paths and performing synchronous cycle rotations; this moves at most $n$ agents along a path of length at most $n$, taking $\mc{O}(n^2)$ time. Thus, clearing and advancing costs $\mc{O}(n^2)$ per component, leading to $\mc{O}(n^3)$ for the entire inner loop. Generating and applying the inverse motion (lines 29-30) traces back this sequence, taking another $\mc{O}(n^3)$. Processing a single agent therefore takes $\mc{O}(n^3)$ time, and over all tasked agents, Phase 1 costs $\mc{O}(|\A^t| n^3) \le \mc{O}(n^4)$. In Phase 2, \textsc{ARouting} (line 33) solves a final pivot-free instance; applying flow-based anonymous routing requires time at most $\mc{O}(n^3)$. Therefore, Phase 1 dominates, and the Baseline Algorithm runs in $\mc{O}(|\A^t| n^3) \le \mc{O}(n^4)$ worst-case time, proving it is strictly polynomial in the instance size.

\subsection{Pivot-Prioritized Planning (PPP)}
\label{sec:ppp}

PPP trades completeness for solution quality and speed, working in \emph{two phases}: tasked agents are first routed to a pivot; the residual configuration is then an anonymous MAPF problem in which all agents move to the stations.

\begin{algorithmSM}[htb]
    \caption{Pivot-Prioritized Planning (PPP)}
    \label{alg:prioritized_planning_tasked_oriented}
    \small
    \begin{algorithmic}[1]
        \REQUIRE Instance $\I = (\G, \P, \S, \A^t, \A^u, \o)$, Priority order $\sigma = (\sigma_1, \dots, \sigma_{|\A^t|})$
        \ENSURE Valid solution $\Pi$, or $\emptyset$ if failure
        
        \STATE $\mc{R} \gets \{ (\o(a), t) \mid a \in \A^t, t \in [0, \infty) \}$
        
        \FOR{$j = 1$ \TO $|\A^t|$}
            \STATE $a \gets \sigma_j$
            \STATE $\mc{R} \gets \mc{R} \setminus \{ (\o(a), t) \mid t \in [0, \infty) \}$
            \STATE $\pi_a \gets \text{ST-A}^*\big(\o(a), \P, \G, \mc{R}, \max\big(\{0\} \cup \{\tau_{\sigma_i}\}_{i<j}) + |\V|\big)$
            \IF{$\pi_a$ is $\emptyset$} 
                \RETURN $\emptyset$ 
            \ENDIF
            \STATE $\tau_a \gets \text{arrival\_time}(\pi_a, \P)$
            \STATE $\mc{R} \gets \mc{R} \cup \text{Reserve}(\pi_a)$
        \ENDFOR
        
        \STATE $\tau_{\max} \gets \max\big(\{0\} \cup \{\tau_a \mid a \in \A^t\}\big)$
        
        \FOR{$T = \tau_{\max}$ \TO $\tau_{\max} + 2 |\V|$}
            \STATE $f \gets \text{MaxFlow}\big(\G_{ext,un}^{(T)}(\I, \mc{R}, \{(\pi_a, \tau_a)\}_{a \in \A^t})\big)$
            \IF{$|f| == |\A|$}
                \STATE $\Pi \gets \text{ReconstructPaths}(f, \{\pi_a\})$
                \RETURN $\Pi$
            \ENDIF
        \ENDFOR
        \RETURN $\emptyset$
    \end{algorithmic}
\end{algorithmSM}

In \emph{Phase 1}, tasked agents sequentially compute conflict-free paths to a pivot in priority order $\sigma$ (lines 2-11), planning as if untasked agents were absent: moving them out of the way is deferred entirely to Phase 2. The start locations of lower-priority tasked agents are static obstacles: line 1 places an infinite-horizon reservation on every tasked origin in the table $\mc{R}$, and line 4 lifts only the current agent's own reservation before its Space-Time A* search \cite{silver2005cooperative} (line 5); $\mc{R}$ stores occupied (vertex, time) and traversed (edge, time) pairs, preventing both vertex and swapping conflicts. If a tasked agent's path traverses the initial location of an untasked agent at some $t > 0$, the crossing simply becomes a constraint on Phase 2, which must move that agent away before time $t$ or fail. To guarantee termination even when a pivot is unreachable, the search horizon is bounded by the latest arrival time of any previously routed agent plus $|\V|$ steps; if no path is found within it, the algorithm fails (lines 6-8). Otherwise, $\pi_a$ and its arrival time $\tau_a$ are stored, and $\mc{R}$ is updated (lines 9-10) with all reservations induced by $\pi_a$, each traversed edge recorded in both directions.
	
In \emph{Phase 2}, all agents must reach the stations in $\S$ without collisions; exploiting anonymity, we model this as a maximum flow (line 14) on a customized time-expanded graph for a horizon $T \ge \tau_{\max} = \max_{a \in \A^t} \tau_a$ ($\tau_{\max} = 0$ if $\A^t = \emptyset$, line 12). As in \citet{yu2013multiagent}, our graph $\G_{ext,un}^{(T)}$ is built from the time-expanded graph of $\G$ for horizon $T$, adding a super-source and a super-sink. We prune all reserved nodes and their incident edges except for the final arrival nodes of tasked agents, namely all $(v,t) \in \mc{R} \setminus J$, where $J = \{(\pi_a(\tau_a),\tau_a) \mid a \in \A^t\}$. This prevents conflicts between the two phases. The super-source connects, through unit-capacity arcs, to the start locations of untasked agents at $t=0$ and to the pivot locations of tasked agents at their arrival times $t=\tau_a$ (the \emph{injection nodes}); the super-sink connects to all station vertices at time $T$.	
If a flow of value $|\A|$ is found (line 15), the paths are extracted (line 16): tasked agents concatenate their Phase 1 and flow paths, untasked agents use the flow path; otherwise the horizon is incremented, failing once the window is exhausted (line 20). The window searched in line 13 is strictly sufficient:

\begin{lemma}
\label{lem:horizon}
If a Phase 2 flow of value $|\A|$ exists for some horizon $T' \ge \tau_{\max}$, then one exists for some $T \le \tau_{\max} + 2|\V| - 1$.
\end{lemma}

\begin{proof}
    We first analyze the temporal extent of the reservations stored in $\mc{R}$ at the conclusion of Phase 1. Initially, line 1 assigns infinite-horizon reservations to the starting locations of all tasked agents. However, during the execution of the loop (lines 3-11), line 4 systematically removes this infinite-horizon reservation for each agent $a \in \A^t$ prior to computing its path. The only reservations subsequently added to $\mc{R}$ (line 10) are those strictly induced by the spatial and temporal traversal of the computed paths $\pi_a$. Because every path $\pi_a$ terminates at time $\tau_a$, and by definition $\tau_{\max} = \max_{a \in \A^t} \tau_a$, it follows that for all $t > \tau_{\max}$, no reservations exist in $\mc{R}$.

    Now, assume that a valid Phase 2 flow of value $|\A|$ exists for some horizon $T' \ge \tau_{\max}$. This implies that the time-expanded graph $\G_{ext,un}^{(T')}$ supports a collision-free routing, meaning that at time $\tau_{\max}$, all $|\A|$ agents (both tasked and untasked) are positioned at distinct vertices in $\G$ without any vertex or swapping conflicts. Let this valid configuration at time $\tau_{\max}$ be denoted as $\c_{\tau_{\max}}$.

    Because $\mc{R}$ imposes no constraints for $t > \tau_{\max}$, the time-expanded graph for $t > \tau_{\max}$ coincides exactly with the unconstrained time-expanded graph of $\G$. Consequently, the routing problem from the intermediate configuration $\c_{\tau_{\max}}$ to the set of stations $\S$ over the interval $[\tau_{\max}, T']$ reduces to a standard, unconstrained anonymous MAPF instance. 

    According to classical bounds on anonymous multi-agent path finding \cite{yu2013multiagent}, if such an instance is solvable, it can be completed in at most $|\A| + |\V| - 1$ steps. Given that the number of agents is strictly bounded by the number of vertices ($|\A| \le |\V|$), the number of steps required from $\tau_{\max}$ to reach the stations cannot exceed $2|\V| - 1$. Therefore, if a flow exists for $T'$, a valid flow of value $|\A|$ is guaranteed to exist for a total horizon $T$ satisfying $T \le \tau_{\max} + 2|\V| - 1$.
\end{proof}

Phase 2 feasibility is also monotone in $T$ 
so the returned horizon is the minimum admitting a conflict-free completion of the Phase 1 motions.

\begin{proposition}
\label{prop:ppp sound}
 Given a PS-MAPF instance $\I$, if PPP (Algorithm~\ref{alg:prioritized_planning_tasked_oriented}) returns a set of paths $\Pi \neq \emptyset$, then $\Pi \in Sol(\I)$.
\end{proposition}

\begin{proof}
To prove that $\Pi$ is a solution for $\I$, we must show that $\Pi = \{\pi_i\}_{i\in \A}$ is a feasible motion of length $T$ and that all paths are successful.
First, the algorithm terminates successfully only if an integral maximum flow of value $|\A|$ is found at some horizon $T$. Such a flow decomposes into $|\A|$ arc-disjoint unit paths from the super-source to the super-sink; since each super-source arc has unit capacity and targets a distinct injection node, each unit path is unambiguously assigned to one agent. Every path in $\Pi$ has length exactly $T$: for a tasked agent $a$, the Phase 1 path on $[0..\tau_a]$ is concatenated with its flow path on $[\tau_a..T]$, and the two agree at the handover node $(\pi_a(\tau_a),\tau_a)$, so the concatenation is well defined.

We verify conflict-freedom separately for each kind of pair of path segments. 
(i) \emph{Two Phase 1 paths.} Each Space-Time A* search respects all vertex and edge reservations accumulated in $\mc{R}$ by higher-priority agents, while the origins of lower-priority tasked agents remain statically blocked until their turn; hence, no vertex or swapping conflict arises among Phase 1 paths. 
(ii) \emph{Two Phase 2 paths.} The unit vertex capacities of $\G^{(T)}_{ext}$ make the flow paths vertex-disjoint at every time step, and its edge gadget forbids simultaneous opposite traversals of the same edge, excluding swapping conflicts. 
(iii) \emph{A Phase 1 path of agent $a$ and a Phase 2 path of a distinct agent $b$.} 
Every vertex-time pair occupied by a Phase 1 motion belongs to $\mc R$ and, unless it lies in $J$, is deleted from $\G^{(T)}_{ext,un}$; the flow path of $b$ could therefore meet a Phase 1 motion only at an injection node. An integral flow of value $|\A|$ saturates all $|\A|$ unit-capacity super-source arcs, so each injection node already carries the unit injected by the super-source, and its unit vertex capacity admits no further unit; hence $b$'s flow path visits no injection node other than its own. 
Swapping conflicts between the two phases are excluded because reserved edge traversals are pruned in both directions. Consequently, the only vertex-time pair shared between a Phase 1 motion and a flow path is the handover node of one and the same tasked agent, which is not a conflict. Therefore, $\Pi$ is a feasible motion.    

 Second, we verify that every path is successful. By the initialization of the search spaces and super-source, $\pi_i(0) = \o(i)$ holds for all $i \in \A$. For any untasked agent $i \in \A^u$, the flow in Phase 2 connects its origin at $t = 0$ to the super-sink representing the stations $\S$ at time $T$. Thus, there exists a finite time-at-station $t_\S(\pi_i) \le T < \infty$, making the path successful. 
    For any tasked agent $i \in \A^t$, Phase 1 explicitly routes the agent to a pivot in $\P$ at some time $\tau_i \le \tau_{\max}$. Thus, there exists a finite time-at-pivot $t_\P(\pi_i) \le \tau_{\max} < \infty$. Phase 2 then routes this agent from its pivot location to the super-sink representing $\S$ at time $T$, guaranteeing a finite time-at-station $t_\S(\pi_i) \le T < \infty$. Since all conditions are met, every path in $\Pi$ is successful. 
    
Because $\Pi$ is a feasible motion and all its paths are successful, $\Pi$ is a successful motion. Hence, $\Pi \in Sol(\I)$.
\end{proof}

\textbf{Complexity Analysis.} Let $n=|\V|$, $m=|\E|$. Induction on the line-5 horizon bound yields $\tau_a \le |\A^t| n \le n^2$, so Phase 1 costs $\mc{O}(|\A^t| m (\tau_{\max} + n) \log n)$ over its Space-Time A* searches. Phase 2 tries $\mc{O}(n)$ horizons; for each, $\G_{ext,un}^{(T)}$ has $\mc{O}(mT)$ arcs and a maximum flow needs at most $|\A| \le n$ augmenting-path computations of cost $\mc{O}(mT)$, giving $\mc{O}(n^2 m (\tau_{\max} + n))$ overall, which dominates. Since $\tau_{\max} = \mc{O}(n^2)$, PPP runs in $\mc{O}(n^4 m)$ worst-case time, polynomial in the instance size.

\subsection{SAT-Based Optimal Solver}
\label{subsec:optimal}
Reduction-based algorithms often empirically outperform search-based for classical MAPF in densely occupied environments~\cite{surynek2017,stern2019survey}, the very regime motivating PS-MAPF. Consequently, our optimal solver follows the classical reduction-based strategy of sequentially increasing time horizons~\cite{kautz1992satplan,surynek2017}. We first use BA to determine if $\I$ is solvable and, if so, obtain an upper bound on $T^{\ast}$ which prevents an infinite search. Then, we initialize $T$ with the lower bound of Thm.~\ref{thm:station-makespan-lb} (see below) and construct a pseudo-Boolean formula $\phiT$ that is satisfiable if and only if $\I$ admits a solution with station-makespan at most $T$ (Thm.~\ref{thm:formula-equivalence} below). We query an off-the-shelf solver with $\phiT$, incrementing $T$ and repeating if $\phiT$ is unsatisfiable. The first satisfiable $T$ equals $T^{\ast}$, and the satisfying assignment yields an optimal solution.

The standard ``direct'' encoding \cite{surynek2022survey,stern2019survey} introduces a Boolean variable $\chi_{a,v,t}$ for every agent, vertex, and time step. However, we eliminate the dependency on $|\A|$ by observing that tasked agents are identical to untasked agents once they have visited a pivot, and that constraints need not specify agent identities due to the anonymity of pivots and stations. Therefore, we say an agent is \emph{pivot-oriented} ($q = \P$) if it is tasked and has not yet visited a pivot, and \emph{station-oriented} ($q = \S$) otherwise. We call this encoding \emph{anonymized}, with decision variables $x_{v,t}^q=1$ iff any agent in state $q$ occupies vertex $v$ at time $t \in [0..T]$, and $y_{u,v,t}^q=1$ iff an agent in state $q$ moves from $u$ to $v$ between times $t-1$ and $t \in [1..T]$; the latter is defined for all $(u,v) \in \E$ and for $u = v$ (wait actions). Since reaching a pivot immediately transitions an agent to state $\S$, a pivot-oriented agent never occupies a pivot: for $q=\P$, $x^q_{v,t}$ is defined only for $v \notin \P$ and $y^q_{u,v,t}$ only for $u \notin \P$. Any variable outside its defined domain is $0$. With these variables, $\phiT$ is the conjunction of the constraints in Fig.~\ref{fig:formula}, where $N[v] =$ $\{u \in \V \mid (u,v) \in \E\}$ $\cup \{v\}$ is the closed neighborhood of $v$.

\newlength{\widestCleft}
\settowidth{\widestCleft}{$\displaystyle \bigwedge_{v \notin \o(\A^t\cup\A^u)} \neg \xs{v}{0}$}
\begin{figure}[tb]
    \small
    \setlength{\jot}{1pt}
    \begin{align*}
        &\bigwedge_{v \notin\P}\bigwedge_{t=0}^T \neg \xp{v}{t} \vee \neg \xs{v}{t} \tag{C1a}\label{const:vertex-conflict} \\
        &\bigwedge_{t=1}^T \bigwedge_{(u,v) \in \E} (\neg y_{u,v,t}^{\P} \land \neg y_{u,v,t}^{\S}) \lor (\neg y_{v,u,t}^{\P} \land\neg y_{v,u,t}^{\S}) \tag{C1b}\label{const:swapping-conflict} \\
        &\bigwedge_{t=0}^{T-1} \bigwedge_{v \notin\P} (\xp{v}{t} = \sum_{u \in N[v]} \yp{v}{u}{t+1} ) \tag{C2a}\label{const:po-npv-out} \\
        &\bigwedge_{t=1}^{T} \bigwedge_{v \notin\P} (\xp{v}{t} =\sum_{u \in N[v]\setminus\P} \yp{u}{v}{t}) \tag{C2b}\label{const:po-npv-in} \\
        &\bigwedge_{t=0}^{T-1} \bigwedge_{v \in \V} (\xs{v}{t} = \sum_{u \in N[v]} \ys{v}{u}{t+1} ) \tag{C2c}\label{const:so-out} \\
        &\bigwedge_{t=1}^{T} \bigwedge_{v \notin\P} (\xs{v}{t} = \sum_{u \in N[v]} \ys{u}{v}{t} ) \tag{C2d}\label{const:so-npv-in} \\
        &\bigwedge_{t=1}^{T} \bigwedge_{v \in\P} (\xs{v}{t} = \sum_{u \in N[v]\setminus\P} \yp{u}{v}{t} + \sum_{u \in N[v]} \ys{u}{v}{t} ) \tag{C2e}\label{const:pv-in} \\
        &\makebox[\widestCleft][l]{$\displaystyle \bigwedge_{v \in \o(\A^u)} \xs{v}{0}$} \wedge \bigwedge_{v\in\o(\A^u)\setminus\P} \neg\xp{v}{0} \tag{C3a}\label{const:untasked-origin} \\
        &\makebox[\widestCleft][l]{$\displaystyle \bigwedge_{v \in \o(\A^t)\cap\P} \xs{v}{0}$} \wedge \bigwedge_{v\in\o(\A^t)\setminus\P} (\xp{v}{0} \wedge \neg\xs{v}{0}) \tag{C3b}\label{const:tasked-origin} \\
        &\makebox[\widestCleft][l]{$\displaystyle \bigwedge_{v \notin \o(\A)} \neg \xs{v}{0}$} \wedge \bigwedge_{v\notin\o(\A)\cup\P} \neg \xp{v}{0} \tag{C3c}\label{const:neither-origin} \\
        &\makebox[\widestCleft][l]{$\displaystyle \bigwedge_{v\notin\S}\neg\xs{v}{T}$} \wedge \bigwedge_{v\notin\P}\neg\xp{v}{T}\tag{C3d}\label{const:ending}
    \end{align*}
    \caption{Pseudo-Boolean formula $\phiT$.}
    \label{fig:formula}
\end{figure}

The anonymized encoding cannot represent two agents at the same (state, vertex, time), and the flow conservation constraints~\eqref{const:po-npv-out}--\eqref{const:pv-in} forbid agent paths from merges or splits. This prevents vertex conflicts within each state, and constraint~\eqref{const:vertex-conflict} is the relevant, cross-state/non-pivot case. Constraint~\eqref{const:swapping-conflict} excludes swapping conflicts by forbidding simultaneous, opposite traversals of an edge in any combination of states. Constraints~\eqref{const:po-npv-out}--\eqref{const:pv-in} tie vertex occupancy to agent movement (an agent is at a vertex at time $t$ iff exactly one incoming movement variable equals $1$ at $t$ and exactly one outgoing variable at $t+1$) across pivot-/station-oriented and pivot/non-pivot vertex cases. Notably,~\eqref{const:pv-in} models an agent's transition between states. Since~\eqref{const:so-npv-in} contains no $y^{\P}$ terms, the transition can occur only at a pivot which, combined with~\eqref{const:ending}, forces every pivot-oriented agent through a pivot by time $T$. Constraints~\eqref{const:untasked-origin}--\eqref{const:neither-origin} set the $t=0$ states of untasked origins, tasked origins, and unoccupied vertices. \eqref{const:ending} requires that, at time $T$, all non-station vertices are unoccupied and no agent is pivot-oriented. 

\begin{proposition}\label{prop:anon-forward}
    If $\phiT$ is satisfiable, then there exists $\Pi$ with $M_{\S}^{\I}(\Pi) \le T$.
\end{proposition}
\begin{proof}
    Given a satisfying assignment for $\phiT$, we generate a path $\pi_a$ for each $a\in\A$ as follows. First, set $\pi_a(0)=\o(a)$. Since $\o(a)$ is injective, this gives us $|\A|$ distinct origin locations, i.e. paths. Next, for each $t\in[0..T-1]$, if $\pi_a(t)=u$, find the unique $v$ such that $\yp{u}{v}{t}$ or $\ys{u}{v}{t}$ equals $1$. Then set $\pi_a(t+1)=v$ and repeat until we have identified $\pi_a(T)$.

    To show that $\Pi=\{\pi_a\}_{a\in\A}$ is a solution for $\I$ with $M_{\S}^{\I}(\Pi) \le T$, we will first show that each $\pi_a$ is in fact a valid path from $\pi_a(0)$ to $\pi_a(T)$. Then, we will show that $\Pi$ meets all the criteria for a solution to $\I$ set out in Definition~\ref{def:solution}. As each path in $\Pi$ is of length at most $T$, this achieves the desired result.

    \begin{claim}
        Each $\pi_a$ is a valid path, $\Pi$ is conflict-free.
    \end{claim}
    \noindent
    We first demonstrate that $\Pi$ is free of vertex conflicts. Suppose two distinct paths occupy a vertex $v$ at time $t$. This would require either two agents in different states or two agents in the same state converging on $v$. The former directly violates constraint~\eqref{const:vertex-conflict}. The latter is impossible due to the flow conservation constraints (\ref{const:po-npv-in}, \ref{const:so-npv-in}, \ref{const:pv-in}). Because variables are restricted to $\{0,1\}$ and the vertex variable is constrained to the sum of its incoming edge variables, at most one incoming edge variable can equal $1$ for any given state. Thus paths cannot merge, implying that two paths of the same state cannot occupy the same vertex.
        
    Because we have established that exactly one agent occupies $v$ at $t$, at most one of the outgoing flow constraints (\ref{const:po-npv-out}, \ref{const:so-out}) can be activated. Consequently, constraints~\eqref{const:po-npv-out}-\eqref{const:pv-in} enforce that a vertex variable equals $1$ if and only if exactly one incoming and exactly one outgoing edge variable equal $1$. Thus paths do not branch or spontaneously appear/disappear, making each $\pi_a$ a valid, continuous path.
        
    Finally, suppose a swapping conflict occurs between $u$ and $v$ at time $t$. This requires simultaneous traversals in opposite directions, meaning $(\yp{u}{v}{t} \lor \ys{u}{v}{t}) = 1$ and $(\yp{v}{u}{t} \lor \ys{v}{u}{t}) = 1$, which directly violates constraint~\eqref{const:swapping-conflict}. Thus, $\Pi$ is entirely conflict-free.

    \begin{claim}
        For all $a\in\A$, $\pi_a(0)=\o(a)$ and $\pi_a(T)\in\S$. 
    \end{claim}
    \noindent
    By construction, $\pi_a(0)=\o(a)$ for all agents. Suppose there was $a\in\A$ such that $\pi_a(T)=v$ with $v\notin\S$. Given how we traced $\pi_a$ from the satisfying assignment, this would imply that either $\yp{u}{v}{T}$ or $\ys{u}{v}{T}$ equals $1$ for some $u\in N[v]$. Consequently, either $\xp{v}{T}$ or $\xs{v}{T}$ equals $1$, due to the incoming flow constraints (\ref{const:po-npv-in},\ref{const:so-npv-in},\ref{const:pv-in}). Either case would violate~\eqref{const:ending} since $v\notin\S$, contradicting the assumption that we have a satisfying assignment.

    \begin{claim}
        For all $a\in\A^t$, there exists $p\in\P$ and $t\in[0..T]$ such that $\pi_a(t)=p$. 
    \end{claim}
    \noindent
    Suppose there exists $a \in \A^t$ with $\pi_a(t)\notin\P$ for all $t\in[0..T]$. As $\pi_a(0)$ is set to $\o(a)\notin\P$, we know that $\xp{\o(a)}{0}=1$. Since $a$ never occupies a pivot,~\eqref{const:pv-in} is never ``activated''. That is to say, the flow and vertex variables associated with $a$ are always pivot-oriented, implying that $\xp{v}{T}=1$ where $v=\pi_a(T)$. However, this violates~\eqref{const:ending}, a contradiction.
\end{proof}

\begin{proposition}\label{prop:anon-backward}
    If there exists $\Pi$ with $M_{\S}^{\I}(\Pi) \le T$, then $\phiT$ is satisfiable.
\end{proposition}
\begin{proof}
    To construct a satisfying assignment for $\phiT$, we use each $\pi_a\in\Pi$ to determine which decision variables to set to $1$, with all other variables set to $0$. Using $\tau_a$ to abbreviate the time-at-pivot $t_{\P}(\pi_a)$, assign values as follows:
    \begin{align*}
        \xp{v}{t} = 1 \iff & \exists a \in \A^t \mid \pi_a(t) = v,\ t < \tau_a \\
        \xs{v}{t} = 1 \iff & \exists a \in \A^u \mid \pi_a(t) = v\ \vee \\
        & \exists a \in \A^t \mid \pi_a(t) = v,\ t \ge \tau_a\\
        \yp{u}{v}{t} = 1 \iff & \exists a \in \A^t \mid \pi_a(t-1) = u,\ \pi_a(t) = v,\ t \le \tau_a \\
        \ys{u}{v}{t} = 1 \iff & \exists a \in \A^u \mid \pi_a(t-1) = u,\ \pi_a(t) = v\ \vee \\
        & \exists a \in \A^t \mid \pi_a(t-1) = u,\ \pi_a(t) = v,\ t > \tau_a
    \end{align*}
    \noindent
    We now verify this assignment satisfies all constraints of $\phiT$.
    \begin{claim}
        \eqref{const:vertex-conflict},\eqref{const:swapping-conflict} are satisfied.
    \end{claim}
    \noindent
    Since $\Pi$ has no vertex conflicts, at most one agent occupies $v$ at time $t$ for all $v\in\V$ and $t\in[0..T]$. This single agent cannot be simultaneously tasked and untasked, nor before and after its time-at-pivot, thus our variable mapping will set at most one of $\xp{v}{t}$ and $\xs{v}{t}$ to $1$, satisfying~\eqref{const:vertex-conflict}. Similarly, no two agents simultaneously traverse the same edge $(u,v)$ in opposite directions for all $(u,v)\in\E$. Therefore, the conditions in our mapping for either $\yp{u}{v}{t}$ or $\ys{u}{v}{t}$ to be $1$, and the conditions for either $\yp{v}{u}{t}$ or $\ys{v}{u}{t}$ to be $1$, cannot both be satisfied. Thus at most one of ($\yp{u}{v}{t}$ or $\ys{u}{v}{t}$) and ($\yp{v}{u}{t}$ or $\ys{v}{u}{t}$) can be $1$, satisfying~\eqref{const:swapping-conflict}.

    \begin{claim}
        \eqref{const:po-npv-out}-\eqref{const:pv-in} are satisfied.
    \end{claim}
    \noindent
    As $\Pi$ is a solution, all $\pi_a \in \Pi$ are continuous, conflict-free paths without vertex conflicts. Therefore, a vertex $v$ is visited by an agent $a$ at time $t$ if and only if it is preceded by exactly one edge traversal (or wait action) from some $u \in N[v]$ at $t$ (for $t\ge 1$), and succeeded by exactly one edge traversal (or wait action) to some $w \in N[v]$ at $t+1$ (for $t\le T-1$). We show this fact satisfies~\eqref{const:po-npv-out}-\eqref{const:pv-in} by looking at the three cases for the state of vertex $v$ at time $t$.

    First, if no agent occupies $v$ at time $t$, then our mapping will assign $0$ to $\xp{v}{t},\xs{v}{t}$, as well as to any flow variable containing $v$ as an endpoint (incoming at time $t$, outgoing at time $t+1$). Thus both sides of each of \eqref{const:po-npv-out}-\eqref{const:pv-in} are $0$, and satisfied. Second, if a pivot-oriented agent occupies $v\notin\P$ at time $t$, then the mapping sets $\xp{v}{t}=1$ and $\xs{v}{t}=0$. The incoming flow must therefore be pivot-oriented, not station oriented, as $a\in\A^t$ and this traversal happens before $\tau_a$. Thus~\eqref{const:po-npv-in} evaluates to $1=1$ and~\eqref{const:so-npv-in} to $0=0$. The outgoing flow must similarly be pivot oriented, thus~\eqref{const:po-npv-out} and~\eqref{const:so-out} are satisfied at $1=1$,$0=0$ respectively. 
    
    Third, if a station-oriented agent $a$ occupies $v$ at $t$, we set $\xp{v}{t}=0$ and $\xs{v}{t}=1$. For incoming flow, if $v\notin\P$, then the incoming flow was already station-oriented. Thus~\eqref{const:po-npv-in} evaluates to $0=0$ and~\eqref{const:so-npv-in} to $1=1$. If $v\in\P$, either $a\in\A^t$ and~\eqref{const:pv-in} is $1=1+0$ or $a\in\A^u$ and it is $1=0+1$. We know that the pivot-oriented incoming flow ($a\in\A^t$) can only come from non-pivots because otherwise the flow would not be pivot-oriented. Outgoing flow will all be station-oriented (as $a$ is already station-oriented at time $t$), so~\eqref{const:so-out} and~\eqref{const:po-npv-out} are satisfied at $1=1$, $0=0$ respectively.

    \begin{claim}
        \eqref{const:untasked-origin}-\eqref{const:ending} are satisfied.
    \end{claim}
    \noindent
    For any $v \in \o(\A^u)$, an untasked agent begins at $v$, thus our construction assigns $\xs{v}{0} = 1$ and $\xp{v}{0} = 0$, satisfying~\eqref{const:untasked-origin}. For any $v \in \o(\A^t)$, if $v \in \P$, then the agent $a$ with $\o(a)=v$ immediately has $\tau_a=0$. Thus we assign $\xs{v}{0}=1$. If $v \notin \P$, then $\tau_a > 0$ and we set $\xp{v}{0}=1$ and $\xs{v}{0}=0$. In either case,~\eqref{const:tasked-origin} is satisfied. If $v\in\V$ contains no agent at $t=0$, we set $\xp{v}{0} = \xs{v}{0} = 0$, satisfying~\eqref{const:neither-origin}. Finally, for~\eqref{const:ending}, because $M_{\S}^{\I}(\Pi) \le T$, every agent terminates its path at a station $s \in \S$ by time $T$. Thus no agent occupies any $v \notin \S$ at time $T$, so we assign $\xs{v}{T} = 0$ for all of them. Furthermore, since every $a \in \A^t$ reached a pivot at some time $\tau_a \le T$, no $\xp{v}{T}$ will be set to $1$.
\end{proof}

This proves each direction, giving us the desired equivalence.
\begin{theorem}\label{thm:formula-equivalence}
    Let $\I = (\G, \P, \S, \A^t, \A^u, \o)$ be a PS-MAPF instance and $T \in \mathbb{N}$. The formula $\phiT$ is satisfiable if and only if $\I$ admits a solution $\Pi$ with $M_{\S}^{\I}(\Pi) \le T$.
\end{theorem}

While~\eqref{const:po-npv-out}--\eqref{const:pv-in} could be reduced to propositional logic with cardinality constraints~\cite{cai2019exactlyone, frisch2010atmostone} for a traditional SAT solver, doing so may cause a significant increase in the number of clauses. We thus maintain the pseudo-Boolean structure of $\Phi(T)$ and rely on a solver with lazy clause generation, which dynamically generates Boolean conflict clauses from the sums only when necessary~\cite{ohrimenko2026lcg}.

\subsection{Lower Bounds}

The number of variables in the anonymized encoding does not depend on $|\A|$, so we expect it to outperform a direct encoding in dense instances. However, the algorithm's efficiency depends heavily on a tight lower bound for $T^{\ast}$, as every unsatisfiable query is a full SAT call. We derive one from the pivot bottleneck, assuming that $\A^t \neq \emptyset$, otherwise the instance reduces to Anonymous MAPF and can be solved optimally in polynomial time ~\cite{yu2013multiagent}. For each $a\in\A^t$, let $\dminPar{a}=\min_{p\in\P}d(\o(a),p)$, the distance from its origin to its closest pivot. Sort the $\dminPar{a}$ distances in ascending order to create the sequence $(\dmin{1}, \ldots, \dmin{|\A^t|})$.
\begin{lemma}\label{lem:time-at-pivot-lower-bound}
    Let $\I$ be a PS-MAPF instance with $\A^t\neq 0$. For any solution $\Pi$ of $\I$, let $(\tmin{1}, \tmin{2}, \ldots, \tmin{|\A^t|})$ be the time-at-pivots for tasked agents sorted in ascending order. Then, 
    \begin{align*}
        &\tmin{i} \ge \dmin{i} && \forall i\in\{1,2,\ldots|\A^t|\} \\
        &\tmin{i}\ge \tmin{i-|\P|}+1 && \forall i\in\{|\P|+1,\ldots |\A^t|\}.
    \end{align*}
\end{lemma}
\begin{proof}
    For the first statement, consider a fixed $i\in\{1,2,\ldots|\A^t|\}$ and suppose that $\tmin{i}<\dmin{i}$. Certainly the time-at-pivot of any agent $a$ is at least $\dminPar{a}$, as an agent can only travel across one edge per time step. Furthermore, there are $i-1$ time-at-pivots smaller than $\tmin{i}$, by definition. Therefore, $i$ of the $\dminPar{a}$ values are smaller than $\tmin{i}$ (and $\dmin{i}$ as well, by extension). However, this contradicts the definition of $\dmin{i}$.

    For the second statement, observe that the set of pivots can ``process'' at most $|\P|$ agents per time step, else there would be a vertex conflict. Thus $|\P|+1$ agents cannot share the exact same time-at-pivot, i.e. the $i$-th agent to arrive at a pivot must arrive strictly later than the $(i-|\P|)$-th agent.
\end{proof}

These lower bounds depend recursively on the actual time-at-pivots, making them difficult to apply directly. To construct an explicit, independently computable lower bound, we define the sequence $\lb{i}$ for $i \in[1..|\A^t|]$ as $\lb{i} = \dmin{i}$ if $i \le |\P|$, and $\lb{i} = \max(\dmin{i}, \lb{i-|\P|} + 1)$ if $i >|\P|$.

\begin{lemma}\label{lem:lower-bound-sequence}
    Let $\I$ be a PS-MAPF instance with $\A^t\neq 0$. For any solution $\Pi$, let $(\tmin{1}, \tmin{2}, \ldots, \tmin{|\A^t|})$ be the time-at-pivots for tasked agents sorted in ascending order. Then $\tmin{i}\ge\lb{i}$ for all $i\in \{1,2,\ldots, |\A^t|\}$.
\end{lemma}

\begin{proof}
    We proceed by induction on $i$. The base cases ($i \le |\P|$) follow from Lemma~\ref{lem:time-at-pivot-lower-bound} and the fact that $\lb{i}=\dmin{i}$ in this case. Now fix some $j>|\P|$, and assume that $\tmin{i}\ge\lb{i}$ for all $i<j$; we will show that $\tmin{j}\ge\lb{j}$. By Lemma~\ref{lem:time-at-pivot-lower-bound}, we have that $\tmin{j}\ge\dmin{j}$ and $\tmin{j}\ge\tmin{j-|\P|}+1$. As $\tmin{j-|\P|}\ge\lb{j-|\P|}$ by the inductive hypothesis,
    \[\tmin{j}\ge\max(\dmin{j},\lb{j-|\P|}+1)=\lb{j}.\]
\end{proof}

\begin{theorem}\label{thm:station-makespan-lb}
    For every solution $\Pi$ of a PS-MAPF instance $\I$ with $\A^t \neq \emptyset$,
    \[M_{\S}^{\I}(\Pi) \ge\lb{|\A^t|}+\min_{p\in\P,s\in\S}d(p,s).\]
\end{theorem}
\begin{proof}
    First, we show that $\lb{i}$ is an monotone non-decreasing sequence. This implies no element in the sequence is larger than $\lb{|\A^t|}$, thus some tasked agent has a time-at-pivot of at least $\lb{|\A^t|}$. So suppose instead that $i$ is the first index at which $\lb{i}>\lb{i+1}$. The $\dmin{i}$ sequence is sorted, meaning $\dmin{i}\le\dmin{i+1}$, leaving the case that $i>|\P|$. We know that $\lb{i}=\max(\dmin{i}, \lb{i-|\P|} + 1)$ must take the value $\lb{i-|\P|} + 1$, else $\lb{i+1}\ge\dmin{i+1}$ would be greater than or equal to it. That is,
    \[\lb{i-|\P|} + 1=\lb{i}>\lb{i+1}\ge \lb{i+1-|\P|} + 1.\]
    Subtracting one from both sides gives us that $\lb{i-|\P|}>\lb{i+1-|\P|}$, contradicting the assumption that $i$ is the first decreasing index, as desired.

    The tasked agent which first reaches a pivot no earlier than $\lb{|\A^t|}$ must still reach the station at which it terminates, requiring at least $\min_{p\in\P,s\in\S}d(p,s)$ further steps. 
\end{proof}

%% file: Sections/experimental.tex
\section{Experimental Study}
\label{sec:experimental_study}
Our experiments address two questions: on small grid instances where the optimal solver is viable, we map the empirical sources of difficulty of PS-MAPF, the tightness of the lower bound of Thm.~\ref{thm:station-makespan-lb}, and PPP's solution quality; on standard MAPF benchmarks probing those regimes at scale, we compare PPP against BA across five priority orderings and characterize PPP's failure modes. 

The optimal solver was implemented using the C++ interface of CP-SAT~\cite{cpsatlp}, while PPP and BA were programmed in Python 3.11. Experiments were run on a Linux cluster, with each Slurm job allocated 16 cores and 64 GB of RAM. Source code, preprocessing scripts, and generated instances will be made publicly available.

\begin{figure*}[!t]
    \centering
    \begin{subfigure}[t]{0.46\textwidth}
        \centering
        \includegraphics[width=\linewidth]{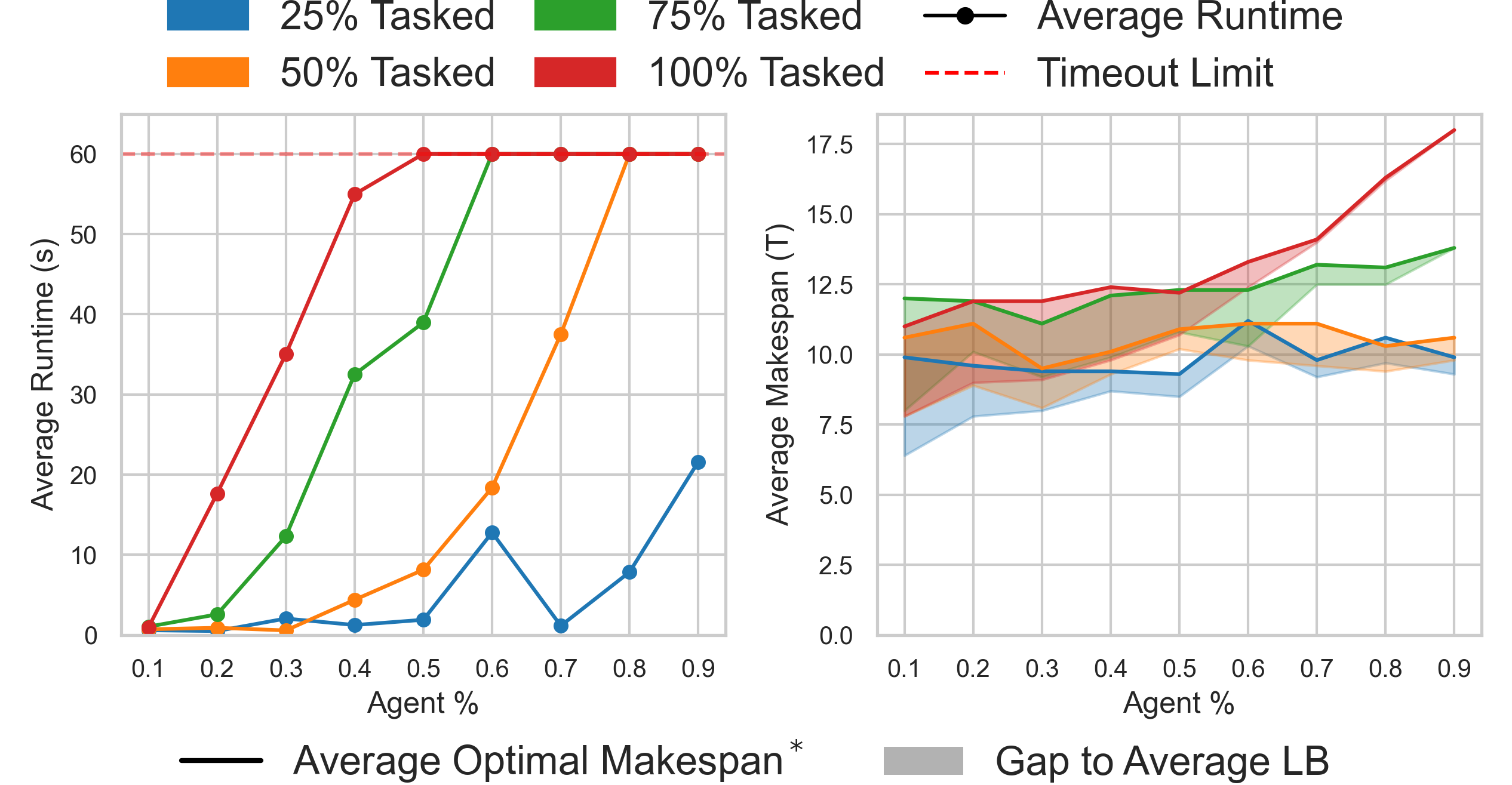}
        \caption{}
        \label{fig:agents_and_tasked}
    \end{subfigure}
    \hfill
    \begin{subfigure}[t]{0.46\textwidth}
        \centering
        \includegraphics[width=\linewidth]{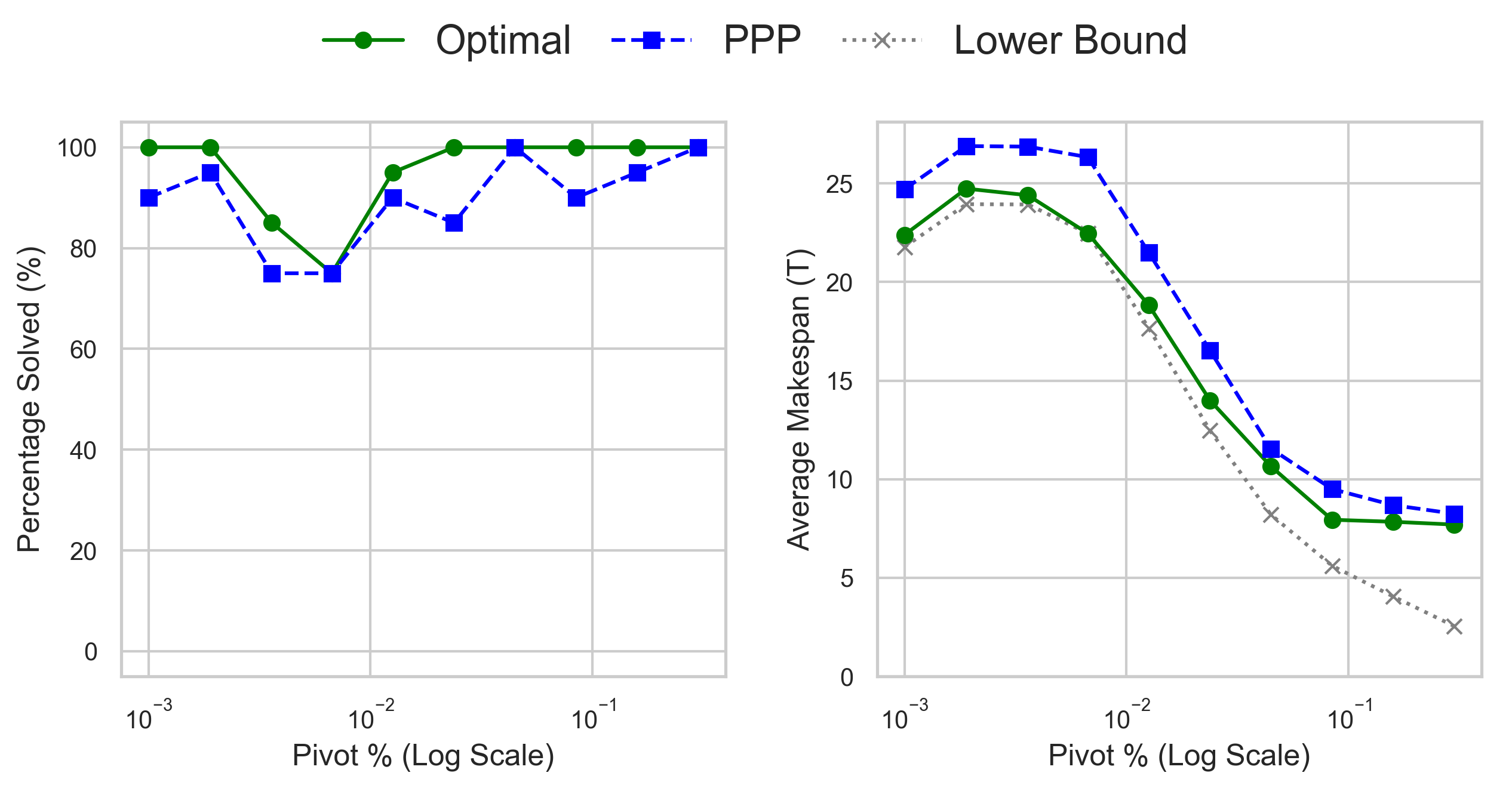}
        \caption{}
        \label{fig:pivots}
    \end{subfigure}
    \caption{(a) Optimal-solver performance over 10 feasible $16\times16$ instances ($\text{Obstacle}\%=0.2$) vs.\ agent density, by tasked percentage ($\text{Station}\%=\text{Agent}\%$, $\text{Pivot}\%=0.05$); (b) Optimal-solver and PPP performances over 20 well-formed $16\times16$ instances ($\text{Obstacle}\%=0.2$) vs.\ pivot density ($\text{Agent}\%=\text{Station}\%=0.15$, $\text{Tasked}\%=0.5$). 60-sec timeout.
    }
    \label{fig:solver-densities}
\end{figure*}

\begin{figure*}[t]
    \centering
    \begin{subfigure}[b]{\linewidth}
        \centering
        \includegraphics[width=0.85\linewidth]{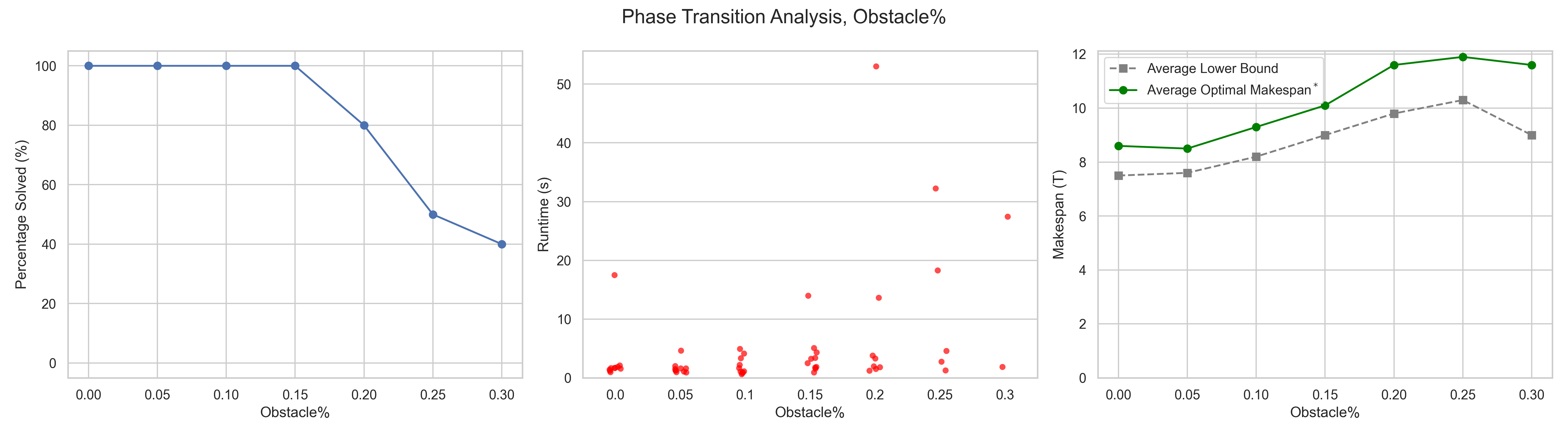}
        \caption{Performance vs Obstacle \%.}
        \label{fig:obstacles}
    \end{subfigure}
    \\[1ex]
    \begin{subfigure}[b]{\linewidth}
        \centering
        \includegraphics[width=0.85\linewidth]{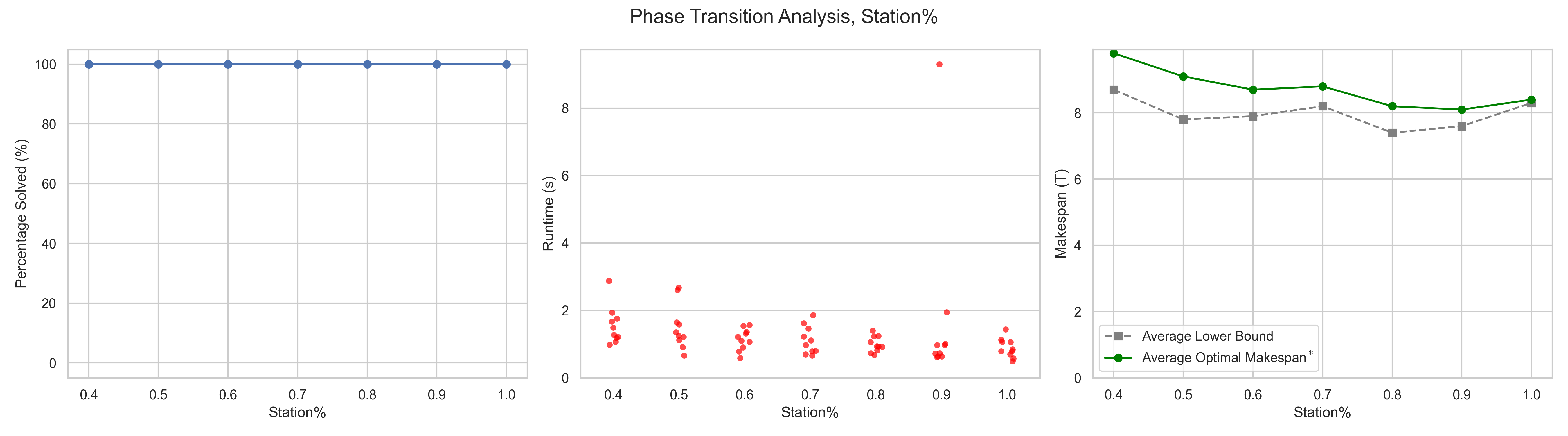}
        \caption{Performance vs Station \%.}
        \label{fig:stations}
    \end{subfigure}
    \\[1ex]
    \begin{subfigure}[b]{\linewidth}
        \centering
        \includegraphics[width=0.85\linewidth]{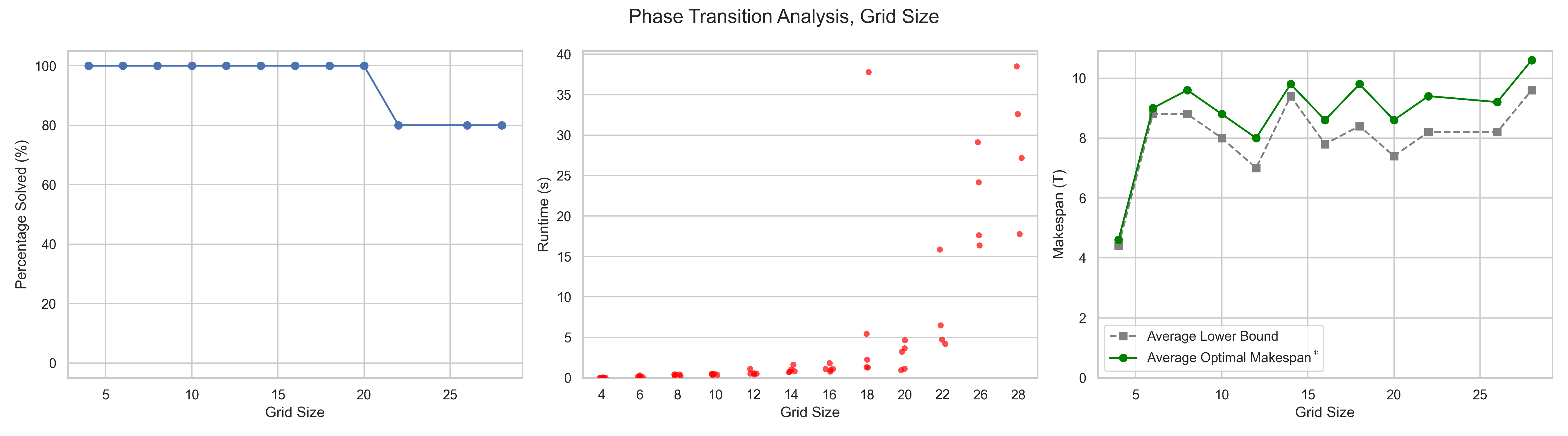}
        \caption{Performance vs Grid Size}
        \label{fig:grid-size}
    \end{subfigure}
    \label{fig:combined_performance}
\end{figure*}

\begin{table*}[!t]
  \centering
  \small
  \setlength{\tabcolsep}{1mm}
  \begin{tabular}{lrrrrrrrrrrrr}
    \toprule
    \multirow{2}{*}{\textbf{Solver}} & \multicolumn{4}{c}{\textbf{empty-48-48}} & \multicolumn{4}{c}{\textbf{random-64-64-10}} & \multicolumn{4}{c}{\textbf{warehouse-10-20-10-2-1}} \\
    \cmidrule(lr){2-5} \cmidrule(lr){6-9} \cmidrule(lr){10-13}
    & \textbf{Runtime} & $\mathbf{M_\S^\I}$ & $\mathbf{\Sigma_\S^\I}$ & \textbf{Success}
    & \textbf{Runtime} & $\mathbf{M_\S^\I}$ & $\mathbf{\Sigma_\S^\I}$ & \textbf{Success}
    & \textbf{Runtime} & $\mathbf{M_\S^\I}$ & $\mathbf{\Sigma_\S^\I}$ & \textbf{Success}\\
    \midrule
    BA    & 22.29 & 3859.72 & 391533.70 & 100.0 & 64.78 & 6326.42 & 750692.14 & 100.0 & 282.32 & 23933.24 & 5740934.78 & 100.0 \\
    \addlinespace
    LH    & 8.76   & 31.33     & 2428.33     & 87.0 & 25.33 & 42.54     & 3855.11     &  84.0 & 141.53   & 91.90       & 17000.67     & \textbf{81.0} \\
    SH    & 8.67   & \textbf{30.88}     & 2387.69      & 86.5   & 25.08 & 42.85     & 3844.63    & \textbf{85.0} & 151.12   & 92.17       & 16681.59     & 78.5\\
    RND & \textbf{8.42}   & 30.98     & \textbf{2365.56}     & 83.5   & 24.70 & 42.35     & 3811.54     & 82.0 & 148.66   & 91.22      & 16876.49     & 78.0 \\
    LC    & 8.91   & 31.17     & 2383.30     & 84.5   & 25.41 & 42.90     & \textbf{3773.30}     & 81.0 & 143.45   & \textbf{90.00}      & \textbf{15983.05}     & 74.5 \\
    MC   & 8.82   & 31.38     & 2452.07     &  \textbf{89.0}   & \textbf{23.44} & \textbf{42.25}    & 3842.34     & 83.5 & \textbf{137.76}   & 91.74      & 16915.78     & 80.5 \\
    \bottomrule
  \end{tabular}
  \caption{Average Runtime (s), Station-Makespan ($M_\S^\I$), Station-Flowtime ($\Sigma_\S^\I$) and Success Rate (\%); for fairness, averages are over instances solved by all six algorithms. Best results among the PPP orderings are in bold. No time limit was imposed.}
  \label{tab:metrics}
\end{table*}

\subsection{The SAT-Based Optimal Solver}
\label{subsec:optimal-experiments}
Instances are generated by randomly assigning vertices on a $16\times 16$ grid to satisfy five parameters: the percentage of grid vertices that are blocked ($\text{Obstacle\%=0.2}$), the percentage of unblocked vertices containing an agent origin (Agent\%), the percentage of tasked agents (Tasked\%), and the percentages of unblocked vertices designated as stations (Station\%, at least Agent\%) and as pivots (Pivot\%). To understand the impact of agent density, we generate 10 feasible instances per configuration, maintaining $\text{Station}\% \ge \text{Agent}\%$ by setting them equal. We evaluate only the optimal algorithm here, as generating the well-formed instances required by PPP becomes prohibitively difficult at high agent densities. For pivot density, we compare the performance of all three solvers and fix $\text{Agent}\%=\text{Station}\%=0.15$ so that it is easier to find well-formed instances for PPP. We sweep pivot density geometrically from 0.1\% to 30\%, generating 20 well-formed instances per data point and excluding instances unsolved by PPP from the makespan averages. We use a 60-sec timeout. 

Instances with a large gap between $T^{\ast}$ and the bound of Thm.~\ref{thm:station-makespan-lb} are likelier to time out; discarding them would inflate the bound's apparent tightness, so unsolved instances are plotted at the highest $T$ reached before timeout.

In the agent-density experiment, Figure~\ref{fig:solver-densities}(a) shows that performance depends on the combination of Agent\% and Tasked\%: runtimes spike at 50\% agent density when all agents are tasked, while instances with 25\% tasked agents remain tractable up to 90\% density. An instance with 90\% density and 25\% tasked agents has significantly higher average runtime than one with 30\% density and 75\% tasked agents, despite both having 22.5\% of unblocked vertices as tasked origins: the number of untasked agents (67.5\% vs.\ 7.5\%) does, in fact, affect difficulty.

We similarly analyze non-agent or pivot parameters. Unsurprisingly, higher obstacle densities increase difficulty. Generating feasible instances becomes prohibitively difficult beyond 30\% obstacles on $16\times16$ grids. Algorithm solve rates hold at 100\% for obstacle densities up to 15\%, but decline sharply to just 40\% at 30\% density. Conversely, increasing station density reduces problem difficulty. With $\text{Obstacle}\%=0.1$ and $\text{Agent}\%=0.4$, we solve all instances for station densities from 40\% to 100\%. Runtimes remain low throughout, and the average optimal makespan decreases linearly from 9.8 (at 40\% density) to 8.4 (at 100\% density). In terms of scalability, our algorithm efficiently solves 100\% of 5 feasible instances on grids up to $20\times20$, and maintains an 80\% success rate on $28\times28$ grids.

An increase in pivot density is most impactful for small $|\P|$ (1--10\% density), as the gains in success rate and the reductions in makespan level off past that point. BA failed to solve any instance within the 60-sec limit, likely bottlenecked by repeated maximum-flow resolutions during its constructive phases. However, PPP excelled across all jointly solved instances, averaging a runtime of just 0.359 sec compared to 14.185 sec for the SAT solver and maintaining an average makespan gap strictly between 0.55 and 2.73 steps.

\subsection{PPP vs. Baseline Algorithm}
\label{subsec:ppp-experiments}
PPP uses a given priority order for the tasked agents, and we evaluate several well-known ordering heuristics \cite{ma2019searching, bennewitz2001optimizing}: \emph{SH (Shortest path first)} prioritizes agents with shorter shortest paths to a pivot, \emph{LH (Longest path first)} the opposite, and \emph{RND (Random)} orders agents uniformly at random (fixed seed 42 for reproducibility). We also evaluate two conflict-based orderings: for each tasked agent, we generate a static, shortest path to a pivot and count the spatio-temporal (vertex and edge) collisions it shares with the others' static paths; \emph{LC (Least Conflicts first)} prioritizes agents with fewer initial conflicts, \emph{MC (Most Conflicts first)} those with more.

We perform experiments on three maps from the Moving AI benchmark \cite{stern2019mapf}: empty $48 \times 48$ grids, $64 \times 64$ grids with 10\% obstacle density, and warehouse maps, covering diverse topologies and congestion levels. For each map, we generated 200 well-formed instances with number of pivots in $\{1, 5, 10, 20, 30\}$, proportion of tasked agents in $\{10\%, 25\%, 50\%, 75\%, 100\%\}$, and total number of agents over eight evenly spaced values, from 1\% of the map's free cells to the maximum multiple of 10 yielding a well-formed instance: 150 (empty), 170 (random), 340 (warehouse).

Table \ref{tab:metrics} presents the evaluated cost metrics and success rates. BA is complete but incurs prohibitive computational and qualitative costs; PPP sacrifices completeness for massive efficiency gains, maintaining a 74--89\% success rate while drastically outperforming BA across all metrics and solving warehouse instances in half the time with orders-of-magnitude lower makespan and flowtime. Among the heuristics, no ordering dominates, but \emph{MC} is the strongest overall, achieving the lowest average runtimes on the two denser maps and the highest success rate on the empty map, while remaining competitive on every metric elsewhere.
PPP's failures concentrate in extreme bottleneck scenarios (e.g., 1--5 pivots, 75--100\% tasked agents): Phase 1 always completes on our instances but saturates $\mathcal{R}$, and this  prevents the Phase 2 flow from routing all agents within the allowed horizon, making all heuristics fail simultaneously at critical density thresholds, though \emph{MC} resolves borderline instances where the others get stuck.

%% file: Sections/conclusions.tex
\section{Conclusion and Future Work}
We introduced PS-MAPF, a MAPF variant in which tasked agents must reach an anonymous pivot before the entire fleet terminates at anonymous stations. We characterized solvability completely, proved that minimizing station-makespan and station-flowtime is \NP-hard already with a single pivot, and presented three algorithms: a complete, low-quality baseline, a SAT-based makespan-optimal solver, and the incomplete but fast Pivot-Prioritized Planning (PPP), with solutions orders of magnitude better than the baseline. Future work includes pivot-based objectives, domain-specific solvers, and adapting state-of-the-art MAPF techniques to PS-MAPF.